%% file: arxiv.tex
\documentclass{article} 

\usepackage[a4paper, total={6in, 8in}]{geometry}

\usepackage{amsmath,amssymb,amsthm,natbib,graphicx,url,dsfont}

\usepackage[utf8]{inputenc} 
\usepackage[T1]{fontenc}    
\usepackage{booktabs}       
\usepackage{amsfonts}       
\usepackage{nicefrac}       
\usepackage{microtype}      
\usepackage{enumitem}
\usepackage{xspace,color,colortbl}
\usepackage{mathrsfs, mathtools}
\usepackage{authblk}

\usepackage{algorithm}
\usepackage[noend]{algorithmic}

\usepackage{thmtools} 
\usepackage{thm-restate}

\usepackage{pgfplots}
\pgfplotsset{compat=1.14}
\usepackage{tikz}
    \usetikzlibrary{shapes,decorations}
    \usetikzlibrary{positioning}
    \usetikzlibrary{cd}
\tikzcdset{every label/.append style = {font = \normalsize}}
\usepackage{adjustbox}

\usepackage{caption}
\usepackage{subcaption}
\usepackage[symbol]{footmisc}
\usepackage[pdfencoding=auto,hidelinks]{hyperref}
\usepackage{times}

\usepackage[capitalize]{cleveref} 

\newcommand{\OPT}{\operatorname{OPT}}
\newcommand{\ALG}{\operatorname{ALG}}

\newcommand{\cC}{\mathcal{C}}
\newcommand{\R}{\mathbb{R}}
\newcommand{\conv}{\mathrm{conv}}

\newcommand{\err}{\operatorname{err}}

\newcommand{\rb}[1]{\left( #1 \right)} 

\newcommand{\Random}{\textsc{Random}\xspace}

\newcommand\poly{\operatorname{poly}}

\newcommand{\1}{\mathds{1}}

\newcommand{\bbR}{\mathbb{R}}
\newcommand{\cA}{\mathcal{A}}

\newcommand{\cI}{\mathcal{I}}

\newcommand{\cF}{\mathcal{F}}

\DeclarePairedDelimiter{\floor}{\lfloor}{\rfloor}
\DeclarePairedDelimiter{\ceil}{\lceil}{\rceil}

\newcommand{\Ot}{\widetilde{O}}

\newcommand{\firstpass}{\textsc{Fair-Reservoir}\xspace}
\newcommand{\secondpass}{\textsc{Fair-Streaming}\xspace}
\newcommand{\OurTwopass}{\textsc{TwoPass-Fair-Streaming}\xspace}
\newcommand{\Matroid}{\textsc{Matroid-Intersection}\xspace}
\newcommand{\Greedy}{\textsc{Greedy-Fair-Streaming}\xspace}
\newcommand{\GreedyMod}{\textsc{Greedy-Fair-Streaming-M}\xspace}
\newcommand{\GreedyFirstpass}{\textsc{Greedy-Fair-Reservoir}\xspace}
\newcommand{\bettersecondpass}{\textsc{Fair-Streaming+}\xspace}

\graphicspath{{./figs/}}

\theoremstyle{plain}
\newtheorem{theorem}{Theorem}[section]
\newtheorem{proposition}[theorem]{Proposition}
\newtheorem{lemma}[theorem]{Lemma}
\newtheorem{corollary}[theorem]{Corollary}
\theoremstyle{definition}
\newtheorem{definition}[theorem]{Definition}

\theoremstyle{remark}

\title{Fairness in Streaming Submodular Maximization over a Matroid Constraint\thanks{We thank Michael Kapralov for helpful discussions. The work of Federico Fusco is partially supported by ERC Advanced Grant 788893 AMDROMA “Algorithmic and Mechanism Design Research in Online Markets”, PNRR MUR project PE0000013-FAIR”, and PNRR MUR project  IR0000013-SoBigData.it. } }
\usepackage{times}

\author[1]{Marwa El Halabi}
\author[2]{Federico Fusco}
\author[3]{Ashkan Norouzi-Fard}
\author[3,4]{Jakab Tardos}
\author[5]{Jakub Tarnawski}

\affil[1]{Samsung - SAIT AI Lab, Montreal}
\affil[2]{Sapienza University of Rome}
\affil[3]{Google Zurich}
\affil[4]{EPFL}
\affil[5]{Microsoft Research}

\date{}
\begin{document}

\maketitle

\begin{abstract}
    Streaming submodular maximization is a natural model for the task of selecting a representative subset from a large-scale dataset. If datapoints have sensitive attributes such as gender or race, it becomes important to enforce fairness to avoid bias and discrimination. This has spurred significant interest in developing fair machine learning algorithms. Recently, such algorithms have been developed for monotone submodular maximization under a cardinality constraint.
    In this paper, we study the natural generalization of this problem to a matroid constraint. We give streaming algorithms as well as impossibility results that provide trade-offs between efficiency, quality and fairness. We validate our findings empirically on a range of well-known real-world applications: exemplar-based clustering, movie recommendation, and maximum coverage in social networks.
\end{abstract}

\clearpage

\section{Introduction}
    Recent years have seen a growing trend of utilizing machine learning algorithms to support or replace human decision-making. An undesirable effect of this phenomenon is the potential for bias and discrimination in automated decisions, especially in sensitive domains such as hiring, access to credit and education, bail decisions, and law enforcement \citep{executive2016,blueprint22,reportEU22}. In order to attenuate such risks, the computer science community has been working on developing \emph{fair} algorithms for fundamental tasks such as classification \citep{ZafarVGG17},  ranking \citep{CelisSV18,SinghJ19}, clustering \citep{Chierichetti0LV17,BackursIOSVW19,BohmFLMS21,JiaSS22,AneggAKZ22,AngelidakisKSZ22}, online learning \citep{JosephKMR16,ChzhenGS21},  voting \citep{CelisHV18}, matching \citep{Chierichetti0LV19}, influence maximization \citep{TsangWRTZ19,RahmattalabiJLV21}, data summarization \citep{CelisKS0KV18}, online selection \citep{CorreaCDN21}, and graph problems \citep{RahmattalabiVFR19,Anagnostopoulos20}.  

    In this paper, we study fairness in the fundamental problem of streaming monotone submodular maximization over a matroid constraint. Submodularity is a well-studied property of set functions that captures the natural notion of diminishing returns and has found vast applications in machine learning, including active learning \citep{GolovinK11}, data summarization \citep{LinB11}, feature selection \citep{DasK11}, and recommender systems \citep{El-AriniG11}. Matroids are a popular and powerful class of independence systems, capturing a wide range of useful constraints such as cardinality, block cardinality, linear independence, and connectivity constraints. In all of the above applications, it is crucial to have the capacity to handle the massive volume of modern datasets, which are often produced so rapidly that they cannot even be stored in memory. This has motivated a long line of work on the streaming setting.
    
    Without considering fairness, maximizing a monotone submodular function over a matroid constraint is a well-studied problem. While a tight $(1-1/e)$-approximation is known in the centralized setting \citep{Calinescu2007,Feige98} and in the multi-pass streaming setting \citep{FeldmanLNSZ22}, the single-pass approximability of the problem is still not settled. Currently, the best one-pass algorithm yields a $0.3178$-approximation \citep{FeldmanLNSZ22}, while the best known upper bound is $0.478$ \citep{GharanV11}.

    Fairness in the context of submodular maximization problems has already been investigated under a \emph{cardinality} constraint,  in both  centralized and streaming models \citep{CelisHV18,HalabiMNTT20}. Defining the notion of algorithmic fairness is an active line of research and, although several notions have been proposed, no universally accepted metric exists. Here, we follow the common notion used in many previous works \citep{CelisHV18, CelisKS0KV18, CelisSV18,Chierichetti0LV17,Chierichetti0LV19}, where we require the solution to be \emph{balanced} with respect to some sensitive attribute (e.g., race, political affiliation). Formally, given a set $V$ of items, each item is assigned a color $c$ representing a sensitive attribute. Let $V_1, \ldots, V_C$ be the corresponding $C$ \emph{disjoint} groups of items of the same color. A set $S \subseteq V$ is \emph{fair} if it satisfies $\ell_c \leq |S \cap V_c| \leq u_c$ for a given choice of lower and upper bounds $\ell_c, u_c \in \mathbb N$. This definition encompasses several other existing notions of fairness such as statistical parity \citep{dwork2012fairness}, diversity rules \citep{cohoon2013, biddle2006adverse}, and proportional representation rules \citep{monroe1995fully, Brill2017}. For a more extended overview we refer to \citet[Section 4]{CelisHV18}.

    \subsection{Our contribution}
    
        We present streaming algorithms as well as impossibility results for the problem of fair matroid monotone submodular maximization (which we abbreviate to FMMSM), that provide trade-offs between memory and computation efficiency, quality and fairness.

        We start by extending the result of \citet{HuangKMY22} to present a $1/2$-approximation algorithm that uses memory exponential with respect to $C$ and $k$, where $k$ is the rank of the input matroid. From the solution quality point of view, this result is tight: the approximation factor of $1/2$ cannot be improved in the streaming setting using memory that is less than linear with respect to $|V|$~\citep{FeldmanNSZ20}. 
        \begin{restatable}{theorem}{algexp}\label{thm:alg-exp}
            For any constant $\eta \in (0,1/2)$, there exists a one-pass streaming $(1/2-\eta)$-approximation algorithm for FMMSM that uses $2^{O(k^2+k\log C)} \cdot \log \Delta$ memory, where $\Delta=\frac{\max_{e\in V} f(e)}{\min_{\{e\in V \mid f(e) > 0\}} f(e)}$.
        \end{restatable}

        The algorithm and its analysis are presented in \cref{app-expo-time}. Motivated by this result, we focus on memory-efficient algorithms, which are referred to as semi-streaming algorithms in the literature. An algorithm is semi-streaming if the memory that it uses is $\tilde{O}(m)$\footnote{We use $\tilde{O}$ notation to hide $\log$ factors, more precisely for any value $m$, $\tilde{O}(m) = m \cdot \poly\log m.$}, where $m$ is the size of the output. We prove that, unlike the cardinality constraint case, it is impossible to design a multi-pass semi-streaming algorithm that even finds a \emph{feasible} solution for a matroid constraint.  

        \begin{restatable}{theorem}{hardnessmultipass} \label{hardness-streaming-main}
            Any (randomized) $o(\sqrt{\log C})$-pass streaming algorithm that determines the existence of a feasible solution for FMMSM with probability at least $2/3$ requires $\max(k, C)^{2-o(1)}$ memory.
        \end{restatable}

        Motivated by \cref{hardness-streaming-main}, we relax the constraints by allowing the fairness lower bounds to be violated by a factor~$2$. More precisely, the goal is to find a solution $S$, feasible with respect to the matroid constraint, that maximizes the value of the submodular function while satisfying $\floor{\ell_c/2} \leq |S \cap V_c| \leq u_c$ for any color $c = 1,\dots,C$. We present a two-pass $1/11.656$-approximation algorithm in this case.

        \begin{restatable}{theorem}{twopass} \label{thm:twopass}
            There exists a two-pass streaming algorithm for FMMSM that runs in polynomial time, uses $O(k \cdot C)$ memory, and outputs a set $S$ such that $(i)$ $S$ is independent, $(ii)$ it holds that $\floor{\ell_c/2} \leq |V_c \cap  S| \le u_c$ for any color $c = 1,\dots,C$, and $(iii)$ $f(S) \ge \mathrm{OPT} / 11.656$.
        \end{restatable}
        Note that although our algorithm is relatively memory-efficient, it is {not} a semi-streaming algorithm.
        Another limitation of our algorithm is that it operates in two passes, instead of a single pass over the stream. We show that at least one of these limitations is necessary, by proving that any one-pass semi-streaming algorithm, which violates the fairness bounds even further than our algorithm, still cannot find a feasible solution.

        \begin{restatable}{theorem}{lowerboundhardness} \label{thm:semistreaming}
            There is no one-pass semi-streaming algorithm that determines the existence of a feasible solution for FMMSM with probability at least $2/3$, even if it is allowed to violate the fairness lower bounds by a factor of $2$ and completely ignore the fairness upper bounds.
        \end{restatable}
        In \Cref{sec:streamingMod,sec:centralized_modular}, we investigate the special case of modular objectives. There, we present efficient exact algorithms for both the streaming and the centralized versions of the problem.

        Finally, we study the performance of our algorithm in multiple real-world experiments: exemplar-based clustering, movie recommendation, and maximum coverage in social networks. We introduce two heuristics that improve the quality of the solution of our two-pass algorithm empirically. Moreover, we present a one-pass heuristic algorithm, based on the ideas of our two-pass algorithm, which is guaranteed to satisfy both the matroid and fairness constraints, but has no worst-case guarantee on the objective value. We observe that our two-pass algorithm achieves similar quality to ``unfair'' baselines, while incurring significantly fewer violations of the fairness constraint. Interestingly, our one-pass heuristic algorithm achieves quality that is not too far from our two-pass algorithm, without violating the fairness constraint.

    \subsection{Related work}

        The problem of fair submodular maximization has already been studied under a \emph{cardinality} constraint. \citet{CelisHV18} provide a tight $(1-1/e)$-approximation to the problem in the centralized setting using a continuous greedy algorithm. The streaming setting has been investigated by \citet{HalabiMNTT20}, who show $(i)$ a $1/2$-approximation one pass algorithm that uses memory exponential in $k$  (this result is tight, see \citet{FeldmanNSZ20}) and $(ii)$ a $1/4$-approximation one pass algorithm, which uses only $O(k)$ memory and processes each element of the stream in $O(\log k)$ time and $2$ oracle calls. 
        
        A closely related problem to FMMSM is monotone submodular maximization over two matroid constraints; FMMSM reduces to this problem when $\ell_c = 0$ for all $c$. \citet{ChakrabartiK15} gave a $1/8$-approximation one-pass streaming algorithm for this problem. The current state-of-the-art \citep{GargJS21} is a $1/5.828$-approximation one-pass  streaming algorithm. 

\section{Preliminaries}

    We consider a ground set $V$ of $n$ items and a non-negative monotone submodular function $f:2^V \to \R_+$. Given any two sets $X,Y \subseteq V$, the marginal gain of $X$ with respect to $Y$ quantifies the change in value of $f$ when adding $X$ to $Y$ and is defined as 
    \[
        f(X\mid Y) = f(X \cup Y) - f(Y). 
    \]
    We use the shorthand $f(x\mid Y)$ for $f(\{x\}\mid Y).$ The function $f$ is submodular if for any two sets $Y \subseteq X \subseteq V$, and any element $e\in V \setminus X$, it holds that $f(e\mid Y) \ge f(e\mid X)$.
    We say that $f$ is monotone if for any element $e \in V$ and any set $Y \subseteq V$ if holds that $f(e\mid Y) \ge 0.$ Throughout the paper, we assume that $f$ is given in terms of a value oracle that computes $f(S)$ for given $S \subseteq V$. We also assume that f is normalized, i.e., $f(\emptyset) = 0$. 
 
    \paragraph{Matroids.}\label{matroids}
    A non-empty family of sets $\cI \subseteq 2^V$ is called a \emph{matroid} if it satisfies the following properties:
    \begin{itemize}
        \item 
            \textit{Downward-closedness}: if $A \subseteq B$ and $B \in \cI$, then $A \in \cI$; 
        \item    
            \textit{Augmentation}: if $A, B \in \cI$ with $|A| < |B|$, then there exists $e \in B$ such that $A + e \in \cI$.
    \end{itemize}
    We write $A + e$ for $A \cup \{e\}$ and $A-e$ for $A \setminus\{e\}$.
    We call a set $A \subseteq V$ \emph{independent} if $A \in \cI$.
    We assume that the matroid is available to the algorithm in the form of an independence oracle.
    An independent set that is maximal with respect to inclusion is called a {\em base}; all the bases of a matroid share the same cardinality $k$, referred to as the {\em rank} of the matroid. An important class of matroids are \emph{partition matroids}, where the universe is partitioned into blocks $V = \bigcup_i V_i$, each with an upper bound $k_i$, and a set $A$ is independent if 
    $|A \cap V_i | \leq k_i$ for all $i$.
    
    A crucial property that follows directly from the definition of a matroid is that given two bases $B_1$ and $B_2$, one can find two elements $b_1\in B_1$ and $b_2 \in B_2$ that can be swapped while maintaining independence, i.e., such that both $B_1 - b_1 + b_2$ and $B_2 - b_2 + b_1$ are independent. This property can be generalized to subsets, see e.g., \citep[Statement 42.31 in Chapter 42]{Schrijver03}:
    \begin{lemma}[Exchange property of bases]
    \label{lem:decomposition}
        In any matroid, for any two bases $B_1$ and $B_2$ and for any partition of $B_1$ into $X_1$ and $Y_1$, there is a partition of $B_2$ into $X_2$ and $Y_2$ such that both $X_1 \cup Y_2$ and $X_2 \cup Y_1$ are bases.
    \end{lemma}

    When two matroids $\cI_1$ and $\cI_2$ are defined on the same set $V$, it is possible to define their intersection $\cI_1 \cap \cI_2$ as the family of the subsets of $V$ that are independent for both matroids. Although the intersection of two matroids is generally not a matroid itself, it is still possible to efficiently compute a maximum-cardinality subset in it, or one of maximum weight when there are weights associated to elements.
    \begin{lemma}[Theorem 41.7 in \citep{Schrijver03}]
    \label{lem:intersection}
        A maximum-weight common independent set in two matroids can be found in strongly polynomial time.    
    \end{lemma}
    
    \paragraph{Fair Matroid Monotone Submodular Maximization (FMMSM) problem.} Recall that each element of $V$ is assigned exactly one color from the set $\{1,...,C\}$; $V_c$ is the set of elements of color $c$.
    We are given fairness bounds $(\ell_c,u_c)_{c = 1,...,C}$ and a matroid $\cI$ on $V$ of rank $k$.
    We denote by $\cF$ the collection of solutions feasible under the fairness and matroid constraints, i.e., 
    \[
        \cF = \{S \subseteq V \mid S \in \cI, \ \ \ell_c \leq |S \cap V_c| \leq  u_c \; \ \forall c = 1, \dots, C\}.
    \]
    Note that we use {\em independent} to denote a set that respects the matroid constraint and {\em feasible} for a set in $\cF$. Clearly, feasibility implies independence, but not vice versa.
    The problem of maximizing a monotone submodular function $f$ under matroid and fairness constraints, which we abbreviate \textbf{FMMSM}, is defined as selecting a set $S \subseteq V$ with $S \in \cF$ to maximize $f(S)$. We use OPT to refer to a feasible set maximizing $f$. We assume a feasible solution exists, i.e., $\cF \neq \emptyset$. In this paper, we study approximations to this problem; in particular, we say that an algorithm is an $\alpha$-approximation to the problem when its output $\ALG$ is in $\cF$ (or possibly, in some relaxed version of $\cF$) and has $ f(\ALG) \ge \alpha \cdot f(\OPT)$. 

\section{Semi-streaming Impossibility Results}  
    In this section we present our impossibility results. We start by showing that even with $O(1)$ passes, finding a feasible solution requires more memory than the semi-streaming setting allows.
    \hardnessmultipass*
    To prove this result, we exploit the fact that it is possible to capture perfect bipartite matching as the intersection of a partition matroid and a fairness constraint.
    We use the following result for streaming bipartite matching, where, given a stream of edges $E$ that belong to a $2n$-vertex bipartite graph $G=(P\cup Q, E)$, the goal is to find a perfect matching. If a perfect matching exists, we say that $G$ is perfectly-matchable.
    \begin{theorem}[Theorem 5.3 in \citep{ChenKPSSY21}] \label{matching-hardness}
        Any randomized $o(\sqrt{\log n})$-pass streaming algorithm that, given a $2n$-vertex bipartite graph  $G(P\cup Q, E)$, determines whether $G$ is perfectly-matchable with probability at least $2/3$, requires $n^{2-o(1)}$ memory.
    \end{theorem}
    \cref{hardness-streaming-main} follows from \cref{matching-hardness} by simply setting up a partition matroid to enforce that every vertex in $P$ has at most one adjacent edge in the solution, and fairness constraints to enforce that every vertex in $Q$ has \emph{exactly} one adjacent edge in the solution (we assign color $q$ to every edge $(p,q)$). We then have $k = |P| = n$ and $C = |Q| = n$. A more detailed proof can be found in \cref{sec:proof_of_hardness-streaming-main}.

    We continue by presenting our second hardness result, which shows that even if we relax fairness lower bounds and ignore fairness upper bounds, nearly-linear memory is still not enough to find any feasible solution in a single pass (let alone one maximizing a submodular function).
     \lowerboundhardness* 
   
    We first state the following auxiliary theorem, which is based on a reduction to the hardness result of \citet{DBLP:journals/corr/abs-2103-11669}. Its proof is provided in \cref{sec:kapralov-proof}. 
    \begin{restatable}{theorem}{thmperfectkapralov} \label{thm:kapralov}
        There is no one-pass semi-streaming algorithm that, given as input the edges of a perfectly-matchable bipartite graph $G=(P \cup Q,E)$, with probability at least $2/3$ finds a matching of size at least $\frac23 |P|$.
    \end{restatable}
    The above theorem shows that it is impossible to approximate the matching problem better than a factor $\frac{2}{3}$ in the semi-streaming model. This result does not yet directly imply \cref{thm:semistreaming}, as in \cref{thm:semistreaming} we allow the fairness bounds to be violated. To handle this, we use the following lemma, which is the key ingredient in our reduction.
    We use $\deg_X(p)$ to denote the degree of a vertex $p$ in the set of edges $X$.
    
    \begin{lemma} \label{lem:semistreaming}
        There is no one-pass semi-streaming algorithm that, given as input the edges of a perfectly-matchable bipartite graph $G=(P \cup Q, E)$, with probability at least $2/3$ finds a set $X \subseteq E$ such that $\deg_X(p) \le 2$ for all $p \in P$ and $\deg_X(q) = 1$ for all $q \in Q$. 
    \end{lemma}
    \begin{proof}
        Suppose towards a contradiction that such an algorithm $\cA$ exists. We use it to design a semi-streaming algorithm for the maximum matching problem as follows:
        \begin{enumerate}
            \item Initialize two copies of the algorithm $\cA$, $\cA'$ (where $\cA'$ will operate on a "flipped" graph whose edges come from $Q \times P$)
            \item When an edge $(p,q)$ arrives:
            \begin{itemize}
                \item pass $(p,q)$ to $\cA$
                \item pass $(q,p)$ to $\cA'$
            \end{itemize}
            \item Let $X$ and $X'$ be the solutions returned by $\cA$ and $\cA'$, respectively
            \item Let $X'' = \{(p,q) : (q,p) \in X'\}$
            \item Return a maximum matching in $X \cup X''$
        \end{enumerate}
        Note that the above algorithm uses $\Ot(n)$ memory (where $n=|P|$).
        We show that it returns a matching of size at least $\frac23 n$,
        which contradicts \cref{thm:kapralov}.
        To see this, assume otherwise. Then, by K\H{o}nig's theorem,
        the graph $(P \cup Q,X \cup X'')$
        contains an independent set $I$ of size larger than $2n - \frac23 n = \frac43 n$.
        It follows that either $|P \cap I| > \frac23 n$ or $|Q \cap I| > \frac23 n$.
        
        We first consider the case $|P \cap I| > \frac23 n$. Focus on the edges in $X$ incident to $Q \cap I$. There are $|Q \cap I|$ many, because $\deg_X(q) = 1$ for all $q \in Q$.
        As $I$ is an independent set also in the graph $(P \cup Q,X)$, all these edges must have their other endpoints in $P \setminus I$.
        We have 
        \[ 
            2 |P \setminus I| = 2n - 2|P \cap I| < \frac43 n - |P \cap I| < |Q \cap I|, 
        \]
        which means that some vertex in $P \setminus I$ must have degree larger than $2$ in $X$,
        contradicting that $\deg_X(p) \le 2$ for all $p \in P$.

        For the case $|Q \cap I| > \frac23 n$ we proceed similarly, swapping the role of $P$ with $Q$ and $X$ with $X''$.
    \end{proof}

    We are now ready to prove \cref{thm:semistreaming}.    
    \begin{proof} [Proof of \cref{thm:semistreaming}]
        We show that if such an algorithm $\cA$ exists, then it can be used to solve the problem from the statement of \cref{lem:semistreaming}. Given any bipartite graph $G=(P \cup Q,E)$, let us define an instance of FMMSM
        on the edges $E$ as follows: the matroid constraint is given by a partition matroid that requires that for a solution $X \subseteq E$ we have
        $\deg_X(p) \le 2$ for each $p \in P$;
        and the color constraints dictate that $\deg_X(q) = 2$ for each $q \in Q$ (that is, an edge $(p,q)$ has color $q$ and we set $\ell_r = u_r = 2$ for all colors $q$).
        
        For each edge arriving on the stream, we pass two copies of it to $\cA$. Then, if we have a feasible instance of the problem from the statement of \cref{lem:semistreaming} (i.e., a perfectly-matchable graph), it gives rise to a feasible instance of FMMSM as in the paragraph above (taking two copies of the perfect matching gives a solution with all vertex degrees equal to $2$). Now, if $\cA$ is an algorithm as in the statement of this theorem, then it returns a solution $X'$ with $\deg_{X'}(p) \le 2$ for each $p \in P$ and $\deg_{X'}(q) \ge 1$ for each $q \in Q$. We obtain $X$ from $X'$ by simply removing, for each $q \in Q$, any $\deg_{X'}(q) - 1$ edges incident to $q$. Then we have $\deg_X(q) \le 2$ for all $p \in P$ and $\deg_X(q) = 1$ for all $q \in Q$, as required by \cref{lem:semistreaming}.
    \end{proof}

    Note that \cref{thm:semistreaming} does not rule out the existence of a \emph{two-pass} semi-streaming algorithm with otherwise the same properties as in its statement.
    However, such an algorithm would give rise to a two-pass semi-streaming $2/3$-approximation for maximum matching, for perfectly-matchable bipartite graphs (using the same arguments as in \cref{thm:semistreaming}). This would significantly improve over the current state of the art, which is a $(2-\sqrt{2}) \approx 0.585$-approximation~\citep{Konrad18}.

\section{Streaming Algorithm}\label{sec:twopass}

   In this section, we present a two-pass algorithm for FMMSM. In particular, we show how to transform any $\alpha$-approximation for streaming submodular maximization over the intersection of two matroids into an $\alpha/2$-approximation for FMMSM, at the cost of a factor-2 violation of the fairness lower bound constraints. Finally, we show that the problem can be solved exactly in one-pass in the special case of modular objectives.

    \subsection{First pass: finding a feasible set}

    The algorithm for the first pass, \firstpass, is simple: it collects a maximal independent set $I_c$ (with respect to $\cI$) for each color \emph{independently}. The number of kept elements is at most $C\cdot k$
    (recall that $k$ is the rank of the matroid $\cI$).
    Then it computes a feasible solution in $\bigcup_c I_c$ as follows. First, it defines the partition matroid $\cI_C$ on $V$ as:
    \begin{equation}\label{def:I_C}
        \cI_C = \{S \subseteq V \mid |V_c \cap S| \le \ell_c \qquad \forall c = 1,\dots C\}.    
    \end{equation}
    Second, it uses any polynomial-time algorithm to find a maximum-size common independent set in $\bigcup_c I_c$ with respect to the two matroids $\cI$ and $\cI_C$. 
    
    To analyze \firstpass we need two ingredients: that a feasible solution is always contained in $\bigcup_c I_c$, and that our algorithm finds one.
    
    \begin{algorithm}
    \caption{\firstpass\label{alg:firstpass}}
        \begin{algorithmic}[1]
            \STATE $I_c \leftarrow \emptyset$ for all $c = 1,...,C$
            \FOR{each element $e$ on the stream}
                \STATE Let $c$ be the color of $e$
                \STATE \textbf{If} $I_c + e \in \cI $ \textbf{then} $I_c \leftarrow I_c + e$  
            \ENDFOR
            \STATE Consider the partition matroid $\cI_C$ on $V$ defined in (\ref{def:I_C})
            \STATE $S \leftarrow$ a max-cardinality subset of $\bigcup_c I_c$ in $\cI \cap \cI_C$  (\Cref{lem:intersection})
            \STATE \textbf{Return} $S$ 
        \end{algorithmic}
    \end{algorithm}
     
    \begin{lemma} \label{lem:existence} 
        For each color $c$, let $I_c \subseteq V_c$ be any maximal subset that is independent with respect to $\cI$. Then, as long as $\cF \neq \emptyset$, there exists a feasible set $R \subseteq \bigcup_c I_c$.
    \end{lemma}

    \begin{proof}
        Let $R$ be any set in $\cF$ such that $|R \setminus \bigcup_c I_c|$ is minimal; note that such $R$ exists as we are assuming that $\cF \ne \emptyset$. To prove the lemma it is enough to show that $|R \setminus \bigcup_c I_c|$ is actually $0$.
    
        Assume towards a contradiction that $|R \setminus \bigcup_c I_c| > 0$. We show how to exchange an element $x \in R \setminus \bigcup_c I_c$ for an element $y \in \bigcup_c I_c \setminus R$ of the same color as $x$ such that $R - x + y \in \cI$. This contradicts the choice of $R$, as $|(R - x + y) \setminus\bigcup_c I_c| = |R \setminus \bigcup_c I_c| - 1$.
    
        Without loss of generality, assume that $(R \cap V_1) \setminus I_1 \neq \emptyset$, and let $x$ be any of its elements. Extend $I_1$ to any maximal independent set $I_1'$ in $I_1 \bigcup R$ containing $I_1$.
        By maximality of $I_1'$, and since $R, I_1' \in \cI$ we have
        \[
            |R - x| < |R| \le |I_1'|.
        \]
        By the matroid augmentation property, there exists $y \in I_1' \setminus (R-x)$ such that $R-x+y \in \cI$. Because \[ I_1' \setminus (R-x) \subseteq (I_1 \cup R) \setminus (R-x) \subseteq I_1 \setminus R + x , \] we must have $y \in I_1 \setminus R$ or $y = x$.
        The latter is impossible; if $y=x$, then $x \in I_1'$ and thus $I_1 + x \subseteq I_1'$;
        as the latter is independent, so is the former. 
        However, as $x \in V_1$, this contradicts the maximality of $I_1$ (as an independent subset of $V_1$). Thus we must have $y \in I_1 \setminus R$ (and recall that $R-x+y \in \cI$), as desired.
    \end{proof}

    \begin{theorem}\label{thm:firstpass}
        There exists a one-pass streaming algorithm that runs in polynomial time, uses $O(k \cdot C)$ memory, and outputs a feasible solution. 
    \end{theorem}
    \begin{proof}
        Any set $I_c$ computed by \firstpass is a maximal subset of $V_c$ that is independent in $\cI$,
        thus \Cref{lem:existence} guarantees the existence of a feasible set $R\subseteq \bigcup_c I_c$.
        By the downward-closeness property of $\cI$, we can further assume that $R$ has exactly $\ell_c$ elements of each color $c$
        (by removing any elements beyond that number), i.e., that $|R| = \sum_c \ell_c$.
        Note that any set independent in $\cI_C$ has at most $\sum_c \ell_c$ elements;
        therefore a maximum-cardinality set $S$ as returned by \firstpass will necessarily have
        $|S| = \sum_c \ell_c$ and thus $|S \cap V_c| = \ell_c$. Hence, $S$ is feasible.
    \end{proof}

    \subsection{Second pass: extending the feasible solution}

    Starting with the solution output by \firstpass (which is feasible but has no guaranteed objective value), we show how to find in another pass a high-value independent set that also respects the fairness constraint, up to some slack in the lower bounds. First, the feasible set $S$ is split into two sets $S_1$ and $S_2$ in a balanced way, i.e.,
    \[
        ||S_1 \cap V_c| - |S_2 \cap V_c||\le 1 \quad \forall c=1,2,\dots, C.
    \]
    Both $S_1$ and $S_2$ are independent in $\cI$ (as subsets of $S$). The goal of the second pass is to extend $S_1$ and $S_2$ by adding elements to them to maximize the submodular function.  To that end, we construct two matroids for each of the sets $S_1$ and $S_2$ as follows.
    First, a partition matroid $\cI^C$ induced by the upper bounds on the colors (note the difference with $\cI_C$, where the partition was induced by the \emph{lower} bounds):
    \begin{equation}\label{eq:color_matroid_upper}
        \cI^C = \{X \subseteq V \mid |X \cap V_c| \le u_c \quad \forall c = 1, \dots, C\}  . 
    \end{equation}
    Second, two matroids $\cI_1$ and $\cI_2$ induced on $\cI$ by $S_1$ and $S_2$:
    \begin{equation}
    \label{eq:matroid_S}
        \cI_i = \{X \subseteq V \mid X \cup S_i \in \cI\}.
    \end{equation}
    It is easy to verify that $\cI_i$ is indeed a matroid.

    Let algorithm $\cA$ be any streaming algorithm that maximizes a monotone submodular function subject to two matroid constraints. We run two parallel independent copies of $\cA$: the first one with matroids $\cI^C, \cI_1$ and the second one with matroids $\cI^C, \cI_2$. Let $S'_1$ and $S'_2$ be the results of these two runs of the algorithm, respectively. We return the solution with the larger value, adding as many elements as necessary from $S_i$ to satisfy the relaxed lower bounds.
    The details of the algorithm are presented in \secondpass.
    \begin{algorithm}
        \caption{\secondpass \label{alg:two-pass}}
        \begin{algorithmic}[1]
            \STATE \textbf{Input:} Set $S$ from \firstpass and routine $\cA$
            \STATE $S_1 \gets \emptyset$, $S_2 \gets \emptyset$ 
            \FOR{$e$ in $S$}
                \STATE Let $c$ be the color of $e$
                \IF{$|S_1 \cap V_c| < |S_2 \cap V_c|$}
                    \STATE $S_1 \gets S_1 + e$
                \ELSE
                    \STATE $S_2 \gets S_2 + e$
                \ENDIF
            \ENDFOR
            \STATE Define matroids $\cI^C$, $\cI_1,\cI_2$ as in \Cref{eq:matroid_S,eq:color_matroid_upper}
            \STATE Run two copies of $\cA$, one for matroids $\cI^C, \cI_1$ and one for matroids $\cI^C, \cI_2$, and let $S'_1$ and $S'_2$ be their outputs
            \FOR{$i=1,2$} \label{l:postprocessing-start}
                \FOR{$e$ in $S_i$}\label{line:fill}
                    \STATE Let $c$ be the color of $e$
                    \STATE \textbf{If} $|S_i'\cap V_c| < u_c$ \textbf{then} $S_i'\gets S_i'+e$
                \ENDFOR
            \ENDFOR \label{l:postprocessing-end}
            \STATE \textbf{Return} $S' = \arg \max (f(S'_1), f(S'_2))$
        \end{algorithmic}
    \end{algorithm}
    We begin the analysis of \secondpass by bounding the violation with respect to upper and lower bounds.

    \begin{lemma}\label{lem:twopass-feasibility}
        The output $S'$ of \secondpass is independent in $\cI$ and for any color $c$ it holds that $\floor{\ell_c/2} \leq |V_c \cap  S'| \le u_c$.
    \end{lemma}
    \begin{proof}  
        Without loss of generality assume that $S' = S_1'$ and divide it into two parts: the elements $X$ that were added by $\cA$, and the elements $Y$ that were added from $S_1$ in the \textbf{for} loop on Line \ref{line:fill}. As $X$ is in $\cI_1$, we have $X \cup S_1 \in \cI$ and therefore also $S_1'=X \cup Y \in \cI$ by downward-closedness.
        
        Consider now the color matroid $\cI^C$ that models the upper bounds $u_c$. As $X$ is in $\cI^C$, and the elements added in the \textbf{for} loop on Line \ref{line:fill} never violate the upper bounds, we have $S_1' \in \cI^C$.
        
        Finally we consider the constraints $\ell_c$ and show that $\floor{\ell_c/2} \leq |V_c \cap  S'|$ for all colors $c$. The set $S$ output by \firstpass is such that $|V_c \cap  S| \geq \ell_c$ and is then divided into $S_1$ and $S_2$ in a balanced way, so that
        \begin{equation}\label{eq:lower_proof}
            |S_1 \cap V_c| \ge \floor{\ell_c/2} \qquad \forall c=1,\dots,C .
        \end{equation} 
        For any color $c$ such that $|S'_1 \cap V_c| < u_c$, all the elements in $S_1 \cap V_c$ are added to $S_1'$, and thus the guarantees on the lower bounds in \eqref{eq:lower_proof} are passed onto $S_1'$.  
    \end{proof}

    \begin{lemma}\label{lem:twopass-approximation}
        Assume that $\cA$ is an $\alpha$-approximate streaming algorithm for the problem of monotone submodular maximization subject to the intersection of two matroids. Then \secondpass is an $\alpha/2$-approximation algorithm.
    \end{lemma}
    \begin{proof}
    Let $S$ be the set output by \secondpass that is then divided into $S_1$ and $S_2$, and let $\OPT$ be the optimal solution. We apply \Cref{lem:decomposition} on the partition $S_1$, $S_2$ of $S$ and $\OPT$.\footnote{\cref{lem:decomposition} is stated for bases of the matroid $\cI$, but it clearly holds also for general independent sets.} Thus $\OPT$ can be partitioned into two sets $O_1$ and $O_2$ such that $O_1 \cup S_1 \in \cI$ and $O_2 \cup S_2 \in \cI$ or, equivalently, $O_i \in \cI_i$ for $i=1,2$. Moreover, both $O_1$ and $O_2$ respect the color matroid $\cI^C$, thus the approximation guarantee of $\cA$ and the monotonicity of $f$ imply that 
    \begin{equation}
        \label{eq:alpha}
          f(S_i') \ge \alpha \cdot f(O_i) \qquad \forall i=1,2.
    \end{equation}
    Now, we are ready to prove the result:
    \begin{align*}
        f(S') &\ge \frac12 \left(f(S'_1) + f(S'_2)\right)\\
        &\ge \frac{\alpha}{2} \left(f(O_1) + f(O_2) \right) \\
        &\ge \frac{\alpha}{2} f(\OPT)
    \end{align*}
    where the first inequality follows by the definition of $S'$, the second by \eqref{eq:alpha}, and the last one by submodularity.
    \end{proof}

    If we plug in the state-of-the-art $1/5.828$-approximation algorithm by \citet{GargJS21} as $\cA$, we get \cref{thm:twopass}:
    \twopass*

    \paragraph{Heuristics.} \label{par:heuristics}
    Although in principle, the feasible solution chosen in the first pass does not need to have any value, empirically it helps to choose a feasible solution with good value. In our empirical evaluation (\cref{sec:exp}), we use an alternative algorithm, \GreedyFirstpass, in the first pass instead of \firstpass. Rather than collecting a maximal independent set $I_c$ of arbitrary elements for each color $c$, \GreedyFirstpass picks elements greedily (see \cref{sec:greedyFirstpass} for details). Similarly, instead of adding arbitrary elements of $S_1, S_2$ at the end (lines \ref{l:postprocessing-start}-\ref{l:postprocessing-end} in \secondpass), we can use $\cA$ to select good elements of $S_1, S_2$ to add (see \cref{sec:betterFilling} for details). We call the resulting algorithm \bettersecondpass.

    We also propose a simple one-pass heuristic streaming algorithm, \Greedy, which runs  \GreedyFirstpass to find a feasible solution, then greedily augments it with elements from $\bigcup_c I_c$  (see \cref{sec:onepassHeuristic} for details).

    \subsection{The modular case}
        We can obtain better results in the special case of modular objectives. In particular,  we present in  \cref{sec:streamingMod} a one-pass algorithm which solves the fair matroid \emph{modular} maximization (F3M) problem exactly. The algorithm greedily collects maximal independent sets $I_c$ for each color $c$ in the same way as in \GreedyFirstpass, then returns an optimal feasible solution in $\cup_{c} I_c$. The second step can be done in polynomial time, as we show in \cref{sec:centralized_modular}, where we present two polynomial time algorithms for the centralized version of F3M. For further details, we refer the reader to \cref{sec:streamingMod,sec:centralized_modular}.
        \begin{restatable}{theorem}{StreamingModular}
            There exists a one-pass streaming algorithm for {F3M}, which finds an optimal solution, uses $O(k \cdot C)$ memory, and runs in polynomial time.
        \end{restatable}

\section{Empirical Evaluation}\label{sec:exp}
    In this section, we empirically evaluate the performance of our algorithms on three applications: maximum coverage, movie recommendation, and exemplar-based clustering, with various choices of fairness and matroid constraints. In comparing our algorithms against two baselines, we consider {objective values}, as well as violations of fairness constraints, which we define for a given set $S$ as $\err(S) = \sum_{c} \max\{|S \cap V_c| - u_c, \ell_c - |S \cap V_c|, 0\}$. Each term in this sum counts the number of elements by which $S$ violates the lower or upper bound. Note that $\err(S)$ is in the range $[0, 2 k]$.
    We compare the following algorithms:
    \begin{itemize}
        \item \textbf{\OurTwopass}: using \GreedyFirstpass in the first pass, and \bettersecondpass in the second pass with $\cA =$ \Matroid, explained below. 
        \item \textbf{\Greedy}: a one-pass heuristic algorithm based on the ideas of our two-pass algorithm (see \cref{sec:onepassHeuristic} for details).
        \item \textbf{\Matroid}: streaming algorithm for submodular maximization over two matroid constraints \citep{ChakrabartiK15} with $\cI$ and $\cI^C$. 
        \item \textbf{\Random}: randomly selects a base set in $\cI$; no fairness constraints.
    \end{itemize}

    We describe below the setup of our experiments. We select fairness bounds $\ell_c, u_c$ which yield instances with feasible solutions, and enforce either that each color group $V_c$ comprises a similar portion of the solution set $S$ (in examplar-based clustering) or that they have a similar representation in $S$ as in the entire dataset (in maximum coverage and movie recommendation). 
    We report the results in \cref{fig:results}, and discuss them in \cref{sec:results}. Varying the specific values of the bounds yields qualitatively very similar results. 
    The code is available at \url{https://github.com/dj3500/google-research/tree/master/fair_submodular_matroid}.

    \input{figs/figure.tex}

        \subsection{Maximum coverage}
        \label{sec:maximum-coverage}
            Given a graph $G = (V, E)$, we aim to select a subset of nodes $S \subseteq V$ that maximize the coverage of $S$ in the graph, given by the monotone submodular function $f(S) = \left|\bigcup_{v \in S} N(v)\right|,$ where $N(v) = \{u : (v, u) \in E\}$ denote the set of neighbors of $v$. 
            We use the Pokec social network~\citep{snapnets}, which consists of 1,632,803 nodes, representing users, and 30,622,564 edges, representing friendships. Each user profile contains attributes such as age, gender, height and weight, which can have ``null'' value.  We impose a partition matroid constraint with respect to body mass index (BMI), which is calculated as the squared ratio between weight (in kg) and height (in m). We discard all profiles with no set height or weight (around $60\%$), as well as profiles with clearly fake data (fewer than $2\%$). The resulting graph has 582,289 nodes and 5,834,695 edges. We partition users into four standard BMI categories (underweight, normal weight, overweight and obese). We set the upper bound for each BMI group $V_i$ to $k_i = \ceil*{ \frac{|V_i|}{|V|} k}$. The resulting rank of the matroid is then roughly $k$.
            We also impose a fairness constraint with respect to age, with 7 groups: $[1, 10], [11, 17], [18, 25], [26, 35], [36, 45], [46+]$, and the last group comprised of records with ``null'' age (around $30\%$). We set $\ell_c = \floor*{0.9 \frac{|V_c|}{|V|} k}$
            and
            $u_c = \ceil*{1.5 \frac{|V_c|}{|V|} k}$. We vary $k$ between $10$ and $200$. 
            The results are shown in \cref{fig:coverage-obj} and \ref{fig:coverage-err}.

        \subsection{Exemplar-based clustering}
        \label{sec:exemplar-based-clustering}
            We consider a dataset containing 4,521 records  of phone calls in a marketing campaign ran by a Portuguese banking institution~\citep{MoroCR14}. We aim to find a representative subset of calls $S \subseteq V$ in order to assess the quality of service. We use $7$ features with numeric values (age, account balance, last contact day, duration, number of contacts during the campaign,  number of days that passed by after the client was last contacted from a previous campaign, number of contacts performed before this campaign) to represent each record $e \in V$ in the Euclidean space as $x_e \in \R^7$. We impose a partition matroid constraint with respect to account balance, with $5$ groups: $(-\infty, 0), [0, 2000), [2000, 4000), [4000, 6000),$ $[6000, \infty)$. We choose equal upper bounds $k_i = k/5$ for each age group $V_i$.  The resulting rank of the matroid is then at most $k$. We also impose a fairness constraint with respect to age, with $6$ groups:  $[0,29], [30,39], [40,49], [50,59], [60,69], [70+]$. We set our fairness bounds as $\ell_c = 0.1 k + 2$ and $u_c = 0.4 k$ for each color $1 \leq c \leq 6$. Then we maximize the following monotone submodular function~\citep{krause2010budgeted}:     
            \[
                f(S) = \sum_{e' \in V} \big(d(e',0)  - \min_{e' \in S \cup \{0\}} d(e',e) \big) \, 
            \]
        where $d(e', e) = \| x_{e'} - x_e \|_2^2$ and $x_0$ is a phantom exemplar, which we choose to be the origin. We vary $k$ between $25$ and $60$. The results are shown in \cref{fig:bank-obj} and \ref{fig:bank-err}.

        \subsection{Movie recommendation}
        \label{sec:movie-recommendation}
            
            We emulate a movie recommendation system using the Movielens 1M dataset~\citep{harper2016movielens}, which includes approximately one million ratings for 3,900 movies by 6,040 users. We implement the experimental design of previous research~\citep{mitrovic2017streaming,norouzi2018beyond,HalabiMNTT20} by computing a low-rank completion of the user-movie rating matrix~\citep{troyanskaya2001missing}, resulting in feature vectors $w_u \in \bbR^{20}$ for each user $u$ and $v_m \in \bbR^{20}$ for each movie $m$. The product $w_u^\top v_m$ approximates the rating of movie $m$ by user $u$. The monotone submodular utility function $f_u(S)$ tailored to user $u$ for a set $S$ of movies is defined as:
            \[ 
            \alpha \cdot \sum_{m' \in M}
            \max\rb{
            \max_{m \in S} \rb{v_m^\top v_{m'}},
            0
            }
            +
            (1 - \alpha)
            \cdot
            \sum_{m \in S} w_u^\top v_m
            . \]
            The parameter $\alpha=0.85$ interpolates between optimizing the coverage of the entire movie collection
            and selecting movies that maximize the user's score.
            We enforce a proportional representation in terms of movie release dates using a laminar matroid with $9$ groups for each decade $d$ between $1911$ and $2000$, and three groups for each 30-year period $t$: 1911--1940, 1941--1970, 1971--2000.
            We set an upper bound of $\ceil*{1.2 \frac{|V_d|}{|V|} k}$ for each decade group $V_d$,
            and an upper bound of roughly $\frac{|V_t|}{|V|} k$ for each 30-year period group $V_t$. The resulting rank of the matroid is then roughly $k$.
            We also partition the movies into 18 genres $c$, which we model using colors. 
            As fairness constraint, we set
            $\ell_c = \floor*{0.8 \frac{|V_c|}{|V|} k}$
            and
            $u_c = \ceil*{1.4 \frac{|V_c|}{|V|} k}$. We vary $k$ between $10$ and $200$. The results are shown in \cref{fig:movies-obj} and \ref{fig:movies-err}.
    
        \subsection{Results} \label{sec:results}
    
            We compare the results of our proposed algorithms, \OurTwopass and \Greedy, with the baselines, \Matroid and \Random -- see \cref{fig:results}. We observe that the value of the submodular function for \OurTwopass and \Greedy is lower than \Matroid by at most  $15\%$ and $26\%$, respectively, while the violation in the fairness constraint is significantly higher for \Matroid. Indeed, \Greedy does not violate the fairness constraint in any of the experiments, as guaranteed theoretically (see \cref{sec:onepassHeuristic}). And the violation of \OurTwopass is often $2-3$ times lower than \Matroid. The objective value of \Random is significantly lower than the other three algorithms on maximum coverage and movie recommendation. Surprisingly, on exemplar-based clustering, \Random obtains  objective values better than \Greedy and \Matroid, and similar to \OurTwopass, for several $k$ values. This comes however at the cost of  significant fairness violations. 
    
    \bibliography{biblio}
    \bibliographystyle{plainnat}

    \clearpage 
    
    \appendix

\section{Heuristics}\label{sec:heuristics}
    In this section, we propose two heuristics which can improve the performance of our two-pass algorithm from \cref{sec:twopass}, and a  one-pass heuristic algorithm for FMMSM.
    
    \subsection{Alternative algorithm for finding a feasible set}\label{sec:greedyFirstpass}
        We propose an alternative algorithm, \GreedyFirstpass, to \firstpass (\cref{alg:firstpass}) for finding a feasible solution $S \in \cF$ in a single pass. \GreedyFirstpass is similar to  \firstpass, but instead of collecting a maximal independent set $I_c$ of arbitrary elements for each color $c$, it picks elements greedily. 
        \begin{algorithm}
        \caption{\GreedyFirstpass\label{alg:Greedyfirstpass}}
            \begin{algorithmic}[1]
                \STATE $I_c \leftarrow \emptyset$ for all $c = 1,...,C$
                \FOR{the next element $e$ on the stream}
                  \STATE Let $c$ be the color of $e$
                    \IF{$I_c + e \in \cI $}
                        \STATE $I_c \leftarrow I_c + e$
                    \ELSE
                        \FOR{$e'\in I_c$ in order of increasing $f(e')$}
                            \IF{$I_c+e-e'\in\cI$ and $f(I_c+e-e')\ge f(I_c)$}
                                \STATE $I_c\gets I_c+e-e'$ and
                                \STATE \textbf{break}
                            \ENDIF
                        \ENDFOR
                    \ENDIF
                \ENDFOR
                \STATE \label{l:feasible-1} Consider the partition matroid $\cI_C$ on $V$ defined in (\ref{def:I_C})
                \STATE \label{l:feasible-2} Let $S \subseteq \cup_{c \in 1\ldots C} I_c$ be any max-cardinality set in $\cI \cap \cI_C$  (\Cref{lem:intersection})
                \STATE \textbf{Return} set $S$ 
            \end{algorithmic}
        \end{algorithm}
    
        Note that \cref{thm:firstpass} still holds;  \GreedyFirstpass is guaranteed to return a feasible solution in polynomial time and using $O(k \cdot C)$ memory.
        Using  \GreedyFirstpass  instead of \firstpass in \secondpass yielded better performance in our empirical evaluation (\cref{sec:exp}).
        Furthermore, in \cref{sec:streamingMod} we show that \GreedyFirstpass can be adapted to return an optimal solution to the problem of fair \emph{modular} maximization over a matroid constraint. 
    
    \subsection{Alternative filling-up procedure}\label{sec:betterFilling}
        \begin{algorithm}
            \caption{\bettersecondpass \label{alg:BetterTwo-pass}}
            \begin{algorithmic}[1]
                \STATE \textbf{Input:} Set $S$ from \firstpass and routine $\cA$
                \STATE $S_1 \gets \emptyset$, $S_2 \gets \emptyset$ 
                \FOR{element $e$ in $S$}
                    \STATE Let $c$ be the color of $e$
                    \IF{$|S_1 \cap V_c| < |S_2 \cap V_c|$}
                        \STATE $S_1 \gets S_1 + e$
                    \ELSE
                        \STATE $S_2 \gets S_2 + e$
                    \ENDIF
                \ENDFOR
                \STATE Define matroids $\cI^C$, $\cI_1,\cI_2$ as in \Cref{eq:matroid_S,eq:color_matroid_upper}
                \STATE Run two copies of $\cA$, one for matroids $\cI^C, \cI_1$ and one for matroids $\cI^C, \cI_2$, and let $S'_1$ and $S'_2$ be the two outputs
                \STATE  Run two copies of $\cA$, one for matroids $\cI^C_1, \cI_0$ and objective $f_1$, and one for matroids $\cI^C_2, \cI_0$ and objective $f_2$, and let $S''_1$ and $S''_2$ be the two outputs
                \STATE \textbf{Return} $S' = \arg \max (f(S'_1 \cup S''_1), f(S'_2 \cup S''_2))$
            \end{algorithmic}
        \end{algorithm}
    We propose an alternative way to augment sets $S_1', S_2'$ at the end of \secondpass (\cref{alg:two-pass},  lines \ref{l:postprocessing-start}-\ref{l:postprocessing-end}).
    Instead of adding arbitrary elements from $S_1, S_2$, we again run $\cA$ with objective $f_i(S) = f(S \cup S_i')$ (which is still monotone submodular), matroid $\cI^C_i = \{ X \subseteq S_i \mid X \cup S_i' \in \cI^C \}$, and a second dummy matroid $\cI_0 = \{ X \subseteq S_i \mid |X| \leq |S_i|\}$, for $i=1, 2$. If $\cA$ outputs a base $S''_i \in \cI^C_i$ (which can be easily enforced), then the output set $S' = \arg \max (f(S'_1 \cup S''_1), f(S'_2 \cup S''_2))$ would still satisfy \cref{lem:twopass-feasibility} and \ref{lem:twopass-approximation}. We refer to this algorithm as \bettersecondpass and it is presented in \cref{alg:BetterTwo-pass}.

    \subsection{One-pass heuristic algorithm}\label{sec:onepassHeuristic}
    
        When $f$ is not modular, we can employ a greedy heuristic to augment the set returned by \GreedyFirstpass to obtain a simple one-pass heuristic, \Greedy, for FMMSM.
    
        \begin{algorithm}
            \caption{\Greedy \label{alg:greedy}}
            \begin{algorithmic}[1]
                \STATE Run \GreedyFirstpass, let $S$ be its output and $I_c$ the set collected for each color $c=1, \cdots, C$
                \FOR{$e \in  \cup_{c \in 1\ldots C} I_c$ in order of decreasing $f(e \mid S)$}
                    \IF{$S + e \in\cI$ and $S + e \in \cI^C$}
                        \STATE $S \gets S + e$
                    \ENDIF
                \ENDFOR
                \STATE \textbf{Return} set $S$  
            \end{algorithmic}
        \end{algorithm}

        Since \Greedy only adds elements to the feasible set $S$ as long as they do not violate the matroid and fairness upper bounds, \cref{thm:firstpass} still holds; \Greedy is guaranteed to output a feasible solution in one pass, using $O(k \cdot C)$ memory. 
        It also performs quite well in terms of objective values in our empirical evaluation (\cref{sec:exp}), though in general it does not provide any worst-case guarantee on the objective value.
    
        On one hand, this is due to the fact that the output $S$ of \GreedyFirstpass is an arbitrary maximum-cardinality set in $\cI \cap \cI_C$ (line \ref{l:feasible-2} in \cref{alg:Greedyfirstpass}); thus it may pick only zero-value elements from $\cup_c I_c$.
        On the other hand, \emph{any} algorithm that collects high-value independent sets $I_c$ in each color without considering the objective-value interactions between these sets,
        and then returns some subset of $\cup_c I_c$,
        is doomed to obtain $O(1/C)$-approximation.
        To see this, consider an instance where each color has two elements: $V_c = \{e_c, e'_c\}$,
        with matching lower and upper bounds $\ell_c = u_c = 1$,
        a matroid $\cI$ that encodes the same constraints as the color upper bounds (i.e., $\cI = \cI^C$),
        and a monotone submodular function $f$ that assigns a value $1$ to every element $e_c$ and
        a value $1.01$ to every element $e'_c$, but in such a way that the $1.01$ value is shared between all elements $e'_c$;
        more formally,
        $f(S) = |S \cap \{ e_c : c=1,...,C\}| + 1.01 \cdot \min(1, |S \cap \{ e'_c : c=1,...,C\}|)$.
        The optimal solution, of value $C$, is to pick all elements $e_c$, but
        \Greedy will pick $I_c = \{e'_c\}$.
        Then there is no subset of $\cup_c I_c$ with value higher than $1.01$,
        which gives a multiplicative gap of $O(1/C)$.

\section{Streaming Modular Case}\label{sec:streamingMod}
    In this section, we present a one-pass algorithm, \GreedyMod, for the fair matroid modular maximization  (F3M) problem in the streaming setting.
    In what follows, we assume that $f$ is modular, but not necessarily monotone.
    \GreedyMod collects maximal independent sets $I_c$ for each color $c$ in the same way as in \GreedyFirstpass, but then it returns an optimal feasible solution in $\cup_{c} I_c$.
    
    \begin{algorithm}
        \caption{\GreedyMod \label{alg:greedyMod}}
        \begin{algorithmic}[1]
            \STATE $I_c \leftarrow \emptyset$ for all $c = 1,...,C$
            \FOR{the next element $e$ on the stream}
                \STATE Let $c$ be the color of $e$
                \IF{$I_c + e \in \cI $}
                    \STATE $I_c \leftarrow I_c + e$
                \ELSE
                    \FOR{$e'\in I_c$ in order of increasing $f(e')$}\label{line:order_m}
                        \IF{$I_c+e-e'\in\cI$ and $f(I_c+e-e')\ge f(I_c)$}\label{line:swap_m}
                            \STATE $I_c\gets I_c+e-e'$
                            \STATE \textbf{break}
                        \ENDIF
                    \ENDFOR
                \ENDIF
            \ENDFOR
            \STATE \label{l:optlinear} Let $S \subseteq \cup_{c \in 1\ldots C} I_c$ be any set in $\cF$ which maximizes $f(S)$.
            \STATE \textbf{Return} set $S$  
        \end{algorithmic}
    \end{algorithm}
    
    We start by recalling the following notions for matroids: A \emph{circuit} of a matroid $\cI$ is a dependent set that is minimal with respect to inclusion. We say that an element $u \in V$ is \emph{spanned} by a set $S \in \cI$ if the maximum size independent subsets of $S$ and $S + u$ are of the same size.  It follows from these definitions that every element $u$ of a circuit $S$ is spanned by $S - u$.
    
    Before proving the result, we need the following lemma which relates independence in matroids with reachability on graphs. Given a directed graph $G = (V,E)$, we denote with $\delta^+(u)$ the out-neighborhood of a vertex $u \in V$ , i.e. $\delta^+(u) = \{v \in V \mid (u,v) \in E\}$.
        \begin{lemma}[Lemma 13 of \citep{FeldmanK018}]
        \label{lem:graph}
            Consider an arbitrary directed acyclic graph $G = (V, E)$ whose vertices are elements of some matroid $\cI$. If every non-sink vertex $u$ of $G$ is spanned by $\delta^+(u)$ in $\cI$, then for every set $S$ of vertices of $G$ which is independent in $\cI$ there must exist an injective function $\psi$ such that, for every vertex $u \in S$, $\psi(u)$ is a sink of $G$ which is reachable from $u$.
        \end{lemma}
    
    Using \Cref{lem:graph}, we can make the following key observation about \GreedyMod.
    
    \begin{lemma}\label{fact:greedy}
        For each color $c$, the independent set $I_c$ output by \GreedyMod maximizes $f$ in $V_c$ over the matroid constraint $\cI$. Moreover, if any element $x\in V_c$ is not in $I_c$, then there must exist a subset $I_c'$ of $I_c$ such that
    \begin{itemize}
        \item $I_c'+x\not\in\cI$,
        \item for all $y\in I_c'$, $f(y)\ge f(x)$.
    \end{itemize}
    \end{lemma}
    \begin{proof}
    For each color $c$,
    let $\OPT_c$ be any independent set in $V_c$ which maximizes $f$ in $V_c$, we show that $f(I_c) \ge f(\OPT_c)$, thus proving the optimality of $I_c$. 
    Consider the following directed graph $G_c = (V_c, E_c)$, whose nodes are the elements in $V_c$. For $t=1, \cdots, |V_c|$, let $x^t$ be the $t$-th element of color $c$ to arrive on the stream, $I_c^t$ the set $I_c$ at time $t$, $Y^t$ the set of elements in $Y^t$ which can be swapped with $x^t$, i.e.,  $Y^t:=\{y \in I^t_c  \mid I^t_c + x^t - y \in \cI\}$. If $x^t$ was swapped with $y^t \in Y^t$, then we add directed edges from $y^t$ to all the elements in $Y^t - y^t + x^t$. If $x^t$ was not added, then we add directed edges from $x^t$ to each element in $Y^t$. If $x^t$ was added without any swap, then its out-neighborhood is empty. 
    Note that every edge in the graph $G_c$ points from a vertex dropped or swapped out at some time to a vertex that is either never deleted or removed at a later time. This time component makes the underlying graph a DAG. Note also that because of the design of the algorithm (lines \ref{line:order_m}-\ref{line:swap_m}), the value of $f$ is non-increasing along every edge $(u, v) \in E_c$, i.e., $f(u) \leq f(y)$.
    
    We want to apply \Cref{lem:graph} on the graph $G_c$. To that end, we argue that for any $t$ where $x^t$ is not a sink, and thus $I_c^t + x^t \not \in \cI$, the set $Y^t + x^t$ is a circuit in $I_c^t + x^t$. First, we show that any proper subset of $Y^t + x^t$ is independent. For any $y \in I_c^t + x^t$, we have that $Y^t + x^t - y \subseteq I_c^t + x^t - y \in \cI$ by the definition of $Y^t$ and since $I_c^t \in \cI$. Hence, $Y^t + x^t - y \in \cI$. Next, we show that $Y^t + x^t$ is a dependent set.
     To see this, assume towards contradiction that this is not the case, i.e. that $Y^t + x^t \in \cI$; then we could repeatedly apply the augmentation property and add to $Y^t$ some elements $\{y_1, y_2, \dots, y_{j}\} \subseteq I_c \setminus Y^t$ until $|Y^t + x^t + y_1 + \dots + y_j| = |I_c|$, while maintaining independence. We get a contradiction: the remaining element $z \in I_c \setminus (Y^t + x^t + y_1 + \dots + y_j)$ satisfies $Y^t + x^t + y_1 + \dots + y_j = I_c + x - z \in \cI$, while on the other hand $z \notin Y^t$, which implies $I_c + x - z \notin \cI$ by definition of $Y^t$. It follows then that for ever non-sink vertex $u$ of $G_c$, its out-neighborhood $\delta^+(u) = Y^t + x^t$ is a circuit, hence $u$ is spanned by $\delta^+(u)$ in $\cI$.
    
     By \Cref{lem:graph}, there exists an injective function $\psi$ which associates each element $u$ in $\OPT_c$ to an element $\psi(u)$ in $I_c$, which is reachable from $u$. As discussed earlier, the value of $f$ is non-increasing along every edge in the graph, and in particular along each $u$-$\psi(u)$ path. Hence $f(u) \le f(\psi(u))$ for all $u \in \OPT_c$, and $f(I_c) \ge f(\OPT_c)$. 
    
    Next, we prove the remaining statement of the lemma. For any $x \in V_c \setminus I_c$,  we define $I_c':= \{y \in I_c \mid I_c + x - y \in \cI\}$. We show that $I_c'$ satisfies the two required properties. First, it is clear that for all $y \in I'_c$ it holds that $f(y) \ge f(x)$, because otherwise the independent set $I_c + x - y$ would have value strictly larger than $I_c$, violating the optimality of $I_c$. Second, using a similar argument as above, we can show that $I_c' + x$ is a circuit, and hence $I_c' + x \notin \cI$.
    \end{proof}
    
    We are now ready to prove the optimality of \GreedyMod.
    \StreamingModular*
    
    \begin{proof}
    For every color $c$, the set $I_c$ collected by \GreedyMod is a maximal independent set in $V_c$; therefore, by \cref{lem:existence} there always exists a feasible set in $\bigcup_cI_c$. We need to prove that when $f$ is modular, $\bigcup_cI_c$ also contains an optimal feasible set. The proof proceeds similarly to that of \cref{lem:existence}. Let $R\in \cF$ be the \emph{optimal} feasible set such that $|R \setminus \bigcup_c I_c|$ is minimal. We will prove that $|R \setminus \bigcup_c I_c|$ is actually 0.
    Assume towards contradiction $|R \setminus \bigcup_c I_c| > 0$. We will show how to exchange an element $x\in R\setminus\cup_cI_c$ for an element $y \in \bigcup_cI_c \setminus R$.  Without loss of generality, assume that $(R \cap V_1) \setminus I_1 \neq \emptyset$, and let $x$ be any of its elements. It is enough to show that there exists an element $y\in I_1 \setminus R$ such that $R-x+y\in\cI$ and $f(R-x+y)\ge f(R)$.
    
    Let $I_1'$ be the set guaranteed by \cref{fact:greedy}. Further let $I_1''$ be a maximal set with $I_1' \subseteq I_1'' \subseteq I_1' \cup R$
    that is independent.
    By maximality of $I_1''$, and since $R, I_1'' \in \cI$ we have
    $|R| \le |I_1''|$ and $|R - x| < |I_1''|$.
    By the matroid augmentation property, there is $y \in I_1'' \setminus (R-x)$ such that $R-x+y \in \cI$.
    Because \[ I_1'' \setminus (R-x) \subseteq (I_1' \cup R) \setminus (R-x) \subseteq I_1' + x , \]
    we must have $y \in I_1' \setminus R$ or $y = x$. The latter is impossible, since this would imply that $x\in I_1''$; however, this is impossible because $I_1'+x$ is not independent by \cref{fact:greedy}. So we have found an element $y\in I_1' \setminus R$ such that $R-x+y\in\cI$ and $f(R-x+y)\ge f(R)$ (by \cref{fact:greedy}). This contradicts the original assumption, and concludes the proof that the output of  \GreedyMod is optimal.
    
    Finally, in the following section, we show that {F3M} can be solved in polynomial time in the offline setting. Hence, line \ref{l:optlinear} in \GreedyMod can be done in polynomial time when $f$ is modular, and hence \GreedyMod runs in polynomial time in this case.
    \end{proof}

\section{Centralized Modular Case}
\label{sec:centralized_modular}
    In this section, we present two polynomial-time algorithms for {F3M}, in the \emph{centralized} setting. One is based on linear programming  
    (Section \ref{sec:LP}) and the other reduces the problem 
    to modular maximization over two matroid constraints (Section \ref{sec:reduction}). Both do not assume monotonicity. 
    
    \subsection{Linear programming algorithm}\label{sec:LP}
    
    Given a modular function $f$, we show that the F3M problem 
    \begin{equation}\label{eq:IP}
     \max_{S \subseteq V} \left\{ f(S) = \sum_{e \in S} f(e) : S \in \cF \right\}
    \end{equation}
can be solved in polynomial time. Let $\1_S \in \R^n$ denote the vector whose $i$-th entry is $1$ if $i \in S$ and $0$ otherwise.
    We consider the exact linear program relaxation of \eqref{eq:IP} given by
    \begin{equation}\label{eq:LP} 
    \max_{x \in [0,1]^n} \left\{ \sum_{e \in V}  x_e f(e) : x \in \conv(\{ \1_S : S \in \cF\} ) \right\}.
    \end{equation} 
    Solving Problem \eqref{eq:LP} is equivalent to solving Problem \eqref{eq:IP}. 
   
  Let $\cI_F$ denote the family of fair sets, i.e.,
    \[
    \cI_F = \{ S \subseteq V : \ell_c \leq |S \cap V_c| \leq  u_c \; \forall c = 1,...,C \} \,. 
    \] 
    Recall that $\cF = \cI \cap \cI_F$. Let $P$ be the matroid polytope of $\cI$ defined as $P_M = \{x \in \R^n_+ : x(A) \leq r(A), \forall A \subseteq V\}$, where $x(A) = \sum_{e \in A} x_e$ , and $r$ is the rank function of $\cI$. The matroid polytope $P$ corresponds to the convex-hull of indicator vectors of independent sets, i.e., $P = \conv(\{ \1_S : S \in \cI\})$.
    
    The following lemma provides the convex-hull of indicator vectors of fair sets.
    \begin{lemma}
    Let $$P_F = \{x \in [0,1]^n : x(V_c) \in [\ell_c, u_c], \forall c= 1,...,C\},$$ then $P_F = \conv(\{ \1_S : S \in \cI_F\})$.
    \end{lemma}
    \begin{proof}
    Since $\{ \1_S : S \in \cI_F\} \subseteq P_F$, then $\conv(\{ \1_S : S \in \cI_F\}) \subseteq P_F$. To prove the other direction, we show that for any $\theta \in \R^n$, the linear program $\max_{x \in P_F} \theta^\top x$ is integral, hence $P_F$ is integral  \citep[Proposition 1.3 in Part III.1, Section 1]{Nemhauser1999}.\\
    Let $V_+$ be the set of indices $i \in V$ where $\theta_i > 0$.
    For each color $c$, let $J_c$ be the set of indices corresponding to the largest $\ell_c$ coefficients $\theta_i$ for $i \in V_c$, and $\bar{J}^+_c$ be the set of indices corresponding to the largest $\min\{u_c - \ell_c, |(V_c \setminus J_c) \cap V_+|\}$ coefficients $\theta_i$ for $i \in V_c \setminus J_c \cap V_+$. Then it is easy to see that the integral vector $x^* = \bigcup_c \1_{J_c} \cup \1_{\bar{J}^+_c}$ is an optimal solution of $\max_{x \in P_F} \theta^\top x$.  Hence, $P_F \subseteq \conv(\{ \1_S : S \in \cI_F\} )$, which concludes the proof.
    \end{proof}

    We show that the convex hull of feasible sets' indicator vectors $\conv(\{ \1_S : S \in \cF\}$ corresponds to the intersection of $P$ and $P_F$. This result generalizes the one given in \citet[Theorems 35 and 45]{edmonds1970} for the intersection of two matroids. Our proof follows a similar structure to the proof given in \citet[Section 41.4]{Schrijver03} of this result. 
    
    We start by showing that the linear system corresponding to $P_F \cap P$ is totally dual integral (TDI), and hence $P_F \cap P$ is integral. We recall first the definitions of TDI and box-TDI.
    
    \begin{definition}[Sections 5.17 and 5.20 in \citep{Schrijver03}]
    A system $M x \leq b$ is called totally dual integral (TDI) if $M$ and $b$ are rational, and the dual of $\max \{ c^\top x : M x \leq b \}$ has an integer optimal solution (if finite), for each $c \in \mathbb{Z}^n_+$.
    Furthermore, a system $M x \leq b$ is called box-totally dual integral (box-TDI) if the system  $M x \leq b, d_1 \leq x \leq  d_2$ is TDI for each $d_1, d_2 \in \mathbb{Z}^n_+$.
    \end{definition}
    
    \begin{theorem}\label{them:box-TDI}
    The linear system $\{\mathbf{0} \leq x \leq \1, x(A) \leq r(A), \forall A \subseteq V, -x(V_c) \leq -\ell_c, x(V_c) \leq u_c, \forall c= 1,...,C\}$ is TDI. Hence, $P_F \cap P$ is integral.
    \end{theorem}
    \begin{proof}
    We first show that the linear system $\{x(A) \leq r(A), \forall A \subseteq V, -x(V_c) \leq -\ell_c, x(V_c) \leq u_c, \forall c= 1,...,C\}$ is box-TDI.
    We can write the linear system as $M x \leq b$. Given any $\theta \in \R^n$, the dual of $\max_x \{\theta^\top x : M x \leq b\}$, is given by:
    \begin{align*}
    \min_{\lambda \geq 0, \alpha \geq 0, \beta \geq 0} \left\{  \sum_{A \subseteq V} \lambda_A r(A) + \sum_{c=1}^C (\alpha_c u_c - \beta_c \ell_c) : \sum_{A \subseteq V} \lambda_A \1_A + \sum_{c=1}^C (\alpha_c  - \beta_c) \1_{V_c} = \theta \right\}
    \end{align*}
    We argue that the dual has an optimal solution $\lambda^*, \alpha^*, \beta^*$ for which the collection of sets $\cC = \{ A \subseteq V: \lambda^*_A >0\}$ form a chain, i.e., if $A, B \in \cC$ then $A \subseteq B$ or $A \subseteq B$. Given any optimal dual solution, let $\delta = \min\{\lambda^*_A, \lambda^*_B\}$, then decrease $\lambda^*_A, \lambda^*_B$ by $\delta$, and increase $\lambda^*_{A \cup B}, \lambda^*_{A \cap B}$ by $\delta$. The modified solution is still feasible since $\ \1_A +  \1_B = \1_{A \cup B} + \1_{A \cap B}$, and it has an equal or lower cost since $r(A) + r(B) \geq r(A \cup B) + r(A \cap B)$. Applying this uncrossing operation for all pairs of sets in $\cC$, results in a chain.
    The submatrix of $M$ with rows corresponding to the constraints $x(A) \leq r(A), \forall A \in \cC$, and $x(V_c) \leq u_c, \forall c= 1,...,C$ is the incidence matrix of the union of two laminar families, hence it is totally unimodular (TU) \citep[Theorem 41.11]{Schrijver03}. Adding the rows corresponding to the constraints $-x(V_c) \leq -\ell_c, \forall c= 1,...,C$ preserves the TU property  \citep[Proposition 2.1 in Part III.1, Section 2]{Nemhauser1999}. It follows then by \citet[Theorem 5.35]{Schrijver03} that the linear system $M x \leq b$ is box-TDI.\\
    By definition of box-TDI, we then have that the linear system corresponding to  $P_F \cap P$ is TDI, which implies that $P_F \cap P$ is integral by \citep[Theorem 5.22]{Schrijver03}.
    \end{proof}

    \begin{corollary}
    We have $\conv(\{ \1_S : S \in \cF\} ) = P_F \cap P$. 
    \end{corollary}
    \begin{proof}
    We note that any integral vector in $P_F \cap P$ must also belong to $\{ \1_S : S \in \cF\} $. Since $P_F \cap P$ is integral (\cref{them:box-TDI}), all its vertices are integral.  Hence $P_F \cap P \subseteq \conv(\{ \1_S : S \in \cF\} )$, and since $P_F = \conv(\{ \1_S : S \in \cI_F\})$ and $P = \conv(\{ \1_S : S \in \cI\})$, we also have $\conv(\{ \1_S : S \in \cF\}) \subseteq P_F \cap P$.
    \end{proof}

    \begin{restatable}{theorem}{linearOffline} \label{thm:linearOffline}
    There is an exact polynomial-time algorithm for {F3M}.
    \end{restatable}
    \begin{proof}
    Since $\conv(\{ \1_S : S \in \cF\} )$ is integral, 
    we can solve problem \eqref{eq:IP} by solving its exact LP relaxation \eqref{eq:LP}.
    The latter can be solved in polynomial time using the ellipsoid method, since $\conv(\cI_M \cap \cI_F)$ admits a polynomial time separation oracle, which simply queries the separation oracles of $P_F$ and $P$.
    \end{proof}
    
    \subsection{Reduction to submodular (modular) maximization over matroid intersection bases}\label{sec:reduction}
    
    In this section we show that fair matroid submodular maximization \textbf{(FMSM)}
    reduces to a version of submodular maximization over an intersection of two matroids (with an extra ``full-rank'' constraint) when $f$ is monotone or modular. This will imply  another polynomial-time exact algorithm for {F3M}. 
    
    Let us define the \textbf{submodular maximization over matroid intersection bases (SMOMIB)} problem as follows:
    \begin{itemize}
    	\item input: two matroids $\cI_1$ and $\cI_2$ on the same ground set $V$, with equal ranks $k_1 = k_2$, and a submodular objective function $f : 2^V \to \mathbb{R}_+$,
    	\item output: a set $S \subseteq V$ that is independent and full-rank in both matroids: $S \in \cI_1 \cap \cI_2$, $|S| = k_1 = k_2$,
    	\item objective: maximize $f(S)$.
    \end{itemize}
    
    \begin{proposition} \label{prop:reduction_smomib}
    Let $\cA$ be an $\alpha$-approximation algorithm to SMOMIB for monotone objectives. Then there is an $\alpha$-approximation to FMMSM.  
    Similarly, if $\cA$ is an $\alpha$-approximation to SMOMIB for modular objectives, then there is an $\alpha$-approximation to FMSM for modular $f$.\footnote{\textbf{Erratum:} In the previous version of this paper, \cref{prop:reduction_smomib} was incorrectly stated for \emph{any} submodular function $f$, while the provided proof only holds for monotone or modular $f$. This revision does not impact any other result in the paper.}
    \end{proposition}
    \begin{proof}
    Let $V = \bigcup_c V_c$, $\cI$, $(\ell_c,u_c)_{c \in C}$, $f$ be an instance of FMSM.
    For every guess $x \in [\sum_c \ell_c, \sum_c u_c]$ we will try to find a good solution of size exactly $x$ using $\cA$.
    
    Let us first sketch the idea: we clone every element $v \in V$ into two elements $v_\ell$ and $v_u$ that are copies (in the sense of the matroids and the function), so that only one of the two can be in a solution, the intuition being that $v_\ell$ is used to satisfy the lower bound on $v$'s color and $v_u$ is used to take elements beyond the lower bound. We can enforce the necessary constraints using a second (laminar) matroid, which will be defined so that a solution of size $x$ must have all its bounds satisfied with equality. We also truncate the first matroid to cardinality $x$, so that the ranks are equal.
    
    Now we formalize the above.
    Fix $x$, and
    let $V' = \{v_\ell, v_u : v \in V\}$
    be the new universe.
    For a set $S' \subseteq V'$ denote its projection $\pi(S')$ to $V$ as
    \[ \pi(S') = \{ v \in V : v_\ell \in S' \text{ or } v_u \in S' \} . \]
    We will define two matroids $\cI_1$ and $\cI_2$ on $V'$.
    Let \[ \cI_1 = \{ S' \subseteq V' : \pi(S') \in \cI \text{ and } (\forall v \in V) \ \{v_\ell, v_u\} \not \subseteq S' \text{ and } |S'| \le x \} . \]
    It is easy to see that $\cI_1$ is a matroid.
    Next, we define $\cI_2$ to be the following laminar matroid:
    \begin{itemize}
    	\item for each color $c$, set $\{v_\ell : v \in V_c\}$ with bound $\ell_c$,
    	\item for each color $c$, set $\{v_u : v \in V_c\}$ with bound $u_c - \ell_c$,
    	\item the union of the latter sets, that is $\{v_u : v \in V\}$, with bound $x - \sum_c \ell_c$.
    \end{itemize}
    Having those two matroids, we verify if each has rank $x$; if not, we skip this guess of $x$.
    
    Finally,
    for the case when $f$ is monotone submodular,
    we define $f' : 2^{V'} \to \mathbb{R}_+$ in the natural way: $f'(S') := f(\pi(S'))$.
    Note that $f'$ is then also monotone submodular.
    Indeed, for any $S' \subseteq V', v' \in V'$, we have $\pi(S') \subseteq \pi(S' \cup \{v'\})$. So
    \begin{align*}
        f'(v' | S') = f(\pi(S' \cup \{v'\}))  - f(\pi(S')) \geq f(\pi(S')) - f(\pi(S')) = 0,
    \end{align*}
    thus $f'$ is monotone.
 Also,
 for any $S' \subseteq T' \subseteq V', v' \in V'$,
  we either have
  $\pi(v') \in \pi(T')$,
  in which case $f'(v' \mid T') = f(\pi(v') \mid \pi(T')) = 0 \le f'(v' \mid S')$ because $f'$ is monotone;
  or $\pi(v') \not \in \pi(T')$,
  in which case $\pi(v') \not \in \pi(S')$ and
  $f'(v' \mid T') = f(\pi(v') \mid \pi(T')) \le f(\pi(v') \mid \pi(S')) = f'(v' \mid S')$ because $f$ is submodular;
    thus $f'$ is submodular too.
    
    Now call $\cA$ on instance $V', \cI_1, \cI_2, f'$. To verify that the reduction works, we need to check:
    \begin{itemize}
    	\item If $S$ is a feasible solution to FMSM, then for guess $x = |S|$, the following ``lift'' $S'$ of $S$ is feasible for SMOMIB: from each color $c$, pick some $\ell_c$ elements $v \in V_c \cap S$ and take $v_\ell$ into $S'$, while taking $v_u$ for the other $|V_c \cap S| - \ell_c$ many elements. Then $\pi(S') = S$ so $S' \in \cI_1$, we also have $S' \in \cI_2$ by construction, and $|S'| = |S| = x = k_1 = k_2$. Also $f(S) = f'(S')$.
    	\item Conversely, for any $x$ and any $S' \in \cI_1 \cap \cI_2$ with $|S'| = x$, we have that $S := \pi(S')$ has $|S| = |S'|$ and one can check that $S$ is feasible for the fair problem (in particular, we must then have $|S' \cap \{v_\ell : v \in V_c\}| = \ell_c$ for all $c$). Also $f(S) = f'(S')$.
    \end{itemize}
    This concludes the proof for the case where $f$ is monotone submodular. 

    If $f$ is a modular function (not necessarily monotone), then we can instead define $f'$ as $f'(S') = \sum_{v : v_\ell \in S'} f(v) + \sum_{v : v_u \in S'} f(v)$, which is also modular. Note that $f'(S') = f(\pi(S'))$ for all sets $S' \in \cI_1$, which is sufficient for the  reduction to hold.
    \end{proof}

    Now we are ready to give another proof of \cref{thm:linearOffline}, which we restate for convenience.

    \linearOffline*
    \begin{proof}
    By \cref{prop:reduction_smomib}, 
    it is enough to give a polynomial-time algorithm for SMOMIB in the special case of modular objective. To that end, we define an objective function $f''(S') = \lambda |S'|  + f'(S')$, where $\lambda$ is sufficiently large.
    This is still a modular function. Now we run any exact weighted matroid intersection algorithm (see \cref{lem:intersection}) to maximize $f''$ over $\cI_1 \cap \cI_2$.  Taking $\lambda > \max_{S' \subseteq V'} f'(S')$ ensures that the returned solution $\hat{S}$ is an optimal solution for modular SMOMIB. 
    Indeed, 
    let $S^*$ be an optimal solution to modular SMOMIB; if $|\hat{S}|$ is not full rank, we have:
    \begin{align*}
     \lambda (k_1-1)  + f'(\hat{S})   \geq \lambda |\hat{S}|  + f'(\hat{S})   \geq \lambda k_1  + f'(S^*). 
    \end{align*}
    Hence, $f'(\hat{S}) \geq f'(S^*) + \lambda > f'(S^*) + f'(\hat{S})$, leading to a contradiction. We thus have that $|\hat{S}| = k_1 = k_2$ and $f'(\hat{S}) \geq f'(S^*)$.
    \end{proof}

\section[Proof of the Theorem]{Proof of \cref{thm:kapralov}} \label{sec:kapralov-proof}
    In this section, we prove \cref{thm:kapralov} which is based on a reduction to the hardness result of \citet[Theorem 1]{DBLP:journals/corr/abs-2103-11669}.
    
    \thmperfectkapralov*
    \begin{proof}
    The main result (Theorem 1) of  \citet{DBLP:journals/corr/abs-2103-11669} states that no single-pass semi-streaming algorithm can find a $((1 / (1 + \ln{2}) + \eta)$-approximate maximum matching in a bipartite graph, for any absolute constant $\eta>0$, with probability at least $1/2$. 
    This differs from the statement of the theorem in two ways: 
    i)  \cref{thm:kapralov} requires the existence of a perfect matching in the input graph which it not the case in \citet[Theorem 1]{DBLP:journals/corr/abs-2103-11669}; ii) the approximation factors are different.
    
    The lower bound of \citet[Theorem 1]{DBLP:journals/corr/abs-2103-11669} is achieved using a hard input distribution on graphs which contain a nearly-perfect matching with high probability. In particular, let $\hat G = (P \cup  Q, \hat E)$ be the random bipartite input graph of the hard distribution and $\hat n = |P|+|Q|$. The definition of $\hat G$ (see Equations (239)-(241) in Section 7.1) and the parameter settings (see (p0)-(p7) in Section 5.2 and Lemma 85) imply that $|P| = N \cdot\Theta(\lfloor L/2 \rfloor + 1), |Q| = N 
    \cdot\Theta(\lceil L/2 \rceil + 1/2),$ where $L$ is an arbitrary, sufficiently large, absolute constant, satisfying $\eta = o(1/L)$, and $N$ is a sufficiently large constant as a function of $L$. We thus have $|P| = (1\pm O(1/L))|Q|$.
     Lemma 150 of \citet{DBLP:journals/corr/abs-2103-11669} states that with probability at least $1-O(1/N)$, $\hat G$ contains a matching of size at least $(1-O(1/L))|P|$.

    Choosing $N, L$ sufficiently large, and $\eta$ sufficiently small, we can ensure that there exists a random distribution of bipartite, $\hat n$-vertex graphs, such that
    \begin{enumerate}
        \item the random graph $\hat G$ has a matching of size at least $0.999 \cdot \hat n/2$ with probability at least $0.999$,
        \item no semi-streaming algorithm can find a $0.6$-approximate maximum matching with probability more than $1/2$.
    \end{enumerate}

    From here, we can exclude the possibility that a semi-streaming algorithm exists that can find a $2/3$-approximate matching, given that the input graph contains a perfect matching. Suppose for contradiction that such an algorithm $\mathcal A$ exists, with $2/3$ success probability \footnote{Assuming any constant success probability here is equivalent, as such an algorithms can always be ran in parallel multiple times, with independent sources of randomness, to boost its success probability}. We can use $\mathcal A$ to solve the hard instance distribution of \citet{DBLP:journals/corr/abs-2103-11669}. We simply augment $\hat G$ with a small number of additional vertices and edges:
    \begin{enumerate}
        \item $\hat n/100$ new vertices added to $P$, called $P^+$;
        \item $|P|+|P^+|-|Q|$ new vertices added to $Q$, called $Q^+$;
        \item a complete bipartite graph between $P$ and $Q^+$;
        \item a complete bipartite graph between $P^+$ and $Q$;
        \item a complete bipartite graph between $P^+$ and $Q^+$.   
    \end{enumerate}
    We call the added edges
    $$E^+=\left(P^+\times Q\right)\cup\left(P\times Q^+\right)\cup\left(P^+\times Q^+\right).$$
    We call the augmented graph
    $$\hat G^+=\left(P\cup P^+,Q\cup Q^+,E\cup E^+\right).$$
    We show that $\hat G^+$ is guaranteed to have a perfect matching with probability at least $0.999$. Let $M_{OPT}$ be the maximum matching in $\hat G$, $P_0$ and $Q_0$ the corresponding unmatched vertices of $P$ and $Q$, respectively. Note that $|P_0| = |P| - |M_{OPT}|$ and $|Q_0| = |Q| - |M_{OPT}|$. We can augment $M_{OPT}$ with edges connecting vertices of $P_0$ to $Q^+$, $Q_0$ to $P^+$, and all the remaining unmatched vertices in $P^+$ to the ones in $Q^+$. To do so we need $|Q^+| = |P^+| - |Q_0| + |P_0| = |P^+| + |P| - |Q|$, which is satisfied. We also need $|P^+| \geq |Q_0| = |Q| - |M_{OPT}|$, which holds with probability at least $0.999$, since  $|M_{OPT}| \geq 0.999 \cdot \hat n/2$ with probability at least $0.999$.
    
    Hence, running $\mathcal A$ on $\hat G^+$ is guaranteed to find a matching of size at least $2/3\cdot |P\cup P^+|$ with probability at least $2/3\cdot 0.999 > 1/2$. We can simply discard edges of $E^+$ from this matching, and still retain a better-than-$0.6$-approximate matching in $\hat G$, leading to a contradiction. To see this note that the number of edges in $E^+$ in the matching returned by $\mathcal A$ is at most $\max\{|P^+|, |Q^+|\}$. So the size of the matching after removing these edges is at least $2/3\cdot |P\cup P^+| - |P^+| - (|P| - |Q|)_+$ which is larger than $0.6 \cdot |P|$.
    \end{proof}

\section{Exponential-Memory Algorithm} \label{app-expo-time}

    In this section we present an algorithm for achieving a nearly $1/2$-approximate solution for FMMSM in the streaming setting, albeit with exponential memory in $k$ and $C$. Our algorithm and proof closely follow the result of \citep{HuangKMY22}.
    
    \algexp*
    
    As is standard technique with exponential-memory streaming algorithms, we will first consider our algorithm to have access to hidden information about some optimal solution. We will then replace decisions based on hidden information with random guessing, and show that our algorithm succeeds with positive probability while consuming a bounded amount of randomness. Finally, we run our algorithm in parallel using all possible sequences of random bits, and conclude that at least one instance of the algorithm succeeds.
    
    Let $\text{OPT}$ be a canonical optimal feasible solution, which appears in the stream in the order $o_1, o_2, \ldots, o_\ell$. We will first present an algorithm that assumes approximate knowledge of $f(\text{OPT})$; specifically we assume that our algorithm receives as input some $\gamma$ where $f(\text{OPT})\in[(1-\eta)\cdot\gamma, \gamma]$ is guaranteed. In the, we will show how to get rid of this assumption at a small cost to memory complexity in terms of the so-called aspect ratio, $\Delta$.
    
    Initially our algorithm will also rely on the following pieces of hidden information:
    \begin{enumerate}
        \item The cardinality $\ell$ of $\text{OPT}$; we call this the \textbf{cardinality oracle}.
        \item The color of any opt element, $c(o_i)$; we call this the \textbf{color oracle}.
        \item The $f$-value of any opt element, conditioned on a set $S$ (that we fix later on), $f(o_i|S)$; we call this the \textbf{function oracle}. Here we need only that the oracle returns the value up to an additive error of $f(\text{OPT})\cdot\eta/\ell$; this will be crucial in bounding the amount of randomness guessing needed to replace the oracle.
        \item The independence in $\mathcal I$ of some set which may contain opt elements, as well as elements form the algorithm's memory, $S\cup\{o_{i_1},\ldots,o_{i_m}\}\overset{?}{\in}\mathcal I$; we call this the \textbf{matroid oracle}.
    \end{enumerate}
    
    With this in mind, the algorithm is presented in \cref{alg:exp}.
    
    \begin{algorithm}
    \caption{Exponential algorithm\label{alg:exp}}
        \begin{algorithmic}[1]
            \STATE \textbf{Input:} Cardinality, color, function, and matroid oracles, and $\gamma$.
            \STATE $\ell\gets|\text{OPT}|$
            \COMMENT{Query the cardinality oracle.}
            \STATE $S\gets\emptyset$
            \FOR{$i\gets1\ldots\ell$}\label{line:main-loop}
                \STATE $c\gets$ color of $o_i$ \COMMENT{Query color oracle.}
                \STATE Set $\theta\in\mathbb Z$ such that $f(o_i|S)\in[\theta\eta\gamma/\ell,(\theta+1)\eta\gamma/\ell)$\label{line:function-oracle} \COMMENT{Query function oracle.}
                \STATE $T\gets\emptyset$
                \FOR{$e$ element in the stream}
                    \IF{$e$ is not color $c$}\label{line:color-check}
                        \STATE \textbf{continue}
                    \ENDIF
                    \IF{$f(e|S)\not\in[\theta\eta\gamma/\ell,(\theta+1)\eta\gamma/\ell)$}\label{line:function-check}
                        \STATE \textbf{continue}
                    \ENDIF
                    \IF{$S\cup T+e\not\in\mathcal I$}\label{line:matroid-precheck}
                        \STATE \textbf{continue}
                    \ENDIF
                    \IF[Query matroid oracle.]{$S\cup\{o_{i+1},\ldots,o_\ell\}+e\not\in\mathcal I$}\label{line:matroid-check}
                        \STATE $T\gets T+e$
                    \ELSE
                        \STATE $s_i\gets e$
                        \STATE $S\gets S+s_i$
                    \ENDIF
                \ENDFOR
                \STATE \textbf{Return} $S$
            \ENDFOR
        \end{algorithmic}
    \end{algorithm}

    Note the invariant that
    \begin{equation}\label{eq:exp-alg}
    \forall i:\ \ \{s_1,\ldots,s_i,o_{i+1},\ldots,o_{\ell}\}\in\mathcal I
    \end{equation}
    is guaranteed by the matroid oracle call on Line~\ref{line:matroid-check}. 
    
    \begin{lemma}\label{lemma:precheck-correct}
    In \cref{alg:exp}, the \textbf{if} clause on Line~\ref{line:matroid-precheck} is satisfied (and thus the algorithm does not proceed to Line~\ref{line:matroid-check}) only if $\{s_1,\ldots,s_{i-1},e,o_{i+1},\ldots,o_\ell\}\not\in\mathcal I$.
    \end{lemma}
    
    \begin{proof}
    
    Consider the first element $e$ for which the Lemma's statement is violated. Recall that at this point $S=\{s_1,\ldots,s_{i-1}\}$ and let $O=\{o_{i+1},\ldots,o_\ell\}$ for simplicity of notation. Suppose for contradiction that $S\cup T+e\not\in\mathcal I$ but $S\cup O +e\in\mathcal I$. Notice also that for all elements $t\in T$ it must be the case that $S\cup O+t\not\in\mathcal I$, otherwise $t$ never would have been added to $T$; this is because, by our assumption, the Lemma statement was true for all previous elements.
    
    If $|T|>|O|$ we immediately get a contradiction: Both $S\cup T$ and $S\cup O$ are independent, and $|S\cup T|>|S\cup O|$ so my the augmentation property of matroids there exists a set $S\cup O+t\in\mathcal I$ for $t\in T$.
    
    If, on the other hand, $|T|\le|O|$, $|S\cup T|<|S\cup +e|$ (also independent), so by the augmentation property of matroids, there exists a set $S\cup T\cup O'\in\mathcal I$ where $O'\subseteq O+e$. However, $e$ cannot be in $O'$, since $S\cup T+e\not\in\mathcal I$ by assumption, so $O'\subseteq O$. Furthermore, $|S\cup T+e|$ and $|S\cup T\cup O'|=|S\cup O+e|>|S\cup O|$. Therefore, by applying the augmentation property again, we can get an independent set of the form $S\cup O+t$ where $t\in T$; this is a contradiction.
    
    \end{proof}

    \begin{lemma}
    \cref{alg:exp} will always find an appropriate element $s_i$, and break out of the loop on Line~\ref{line:main-loop}.
    \end{lemma}
    
    \begin{proof}
    
    We can prove the following stronger claim through induction over $i$: The algorithm will break out of the loop on Line~\ref{line:main-loop} no later than $o_i$'s arrival in the stream.
    
    For any $i$ (both base case and inductive step), we know that $o_i$ is still in the stream when the loop at Line~\ref{line:main-loop} begins. Then the inductive statement follows simply due to the fact that $o_i$ itself will pass all the filters on Lines~\ref{line:color-check},~\ref{line:function-check},~\ref{line:matroid-precheck} (due to \cref{lemma:precheck-correct}), and~\ref{line:matroid-check}: It is the right color, the right size, and $e=o_i$ satisfies the condition $\{s_1,\ldots,s_{i-1},e,o_{i+1},\ldots,o_\ell\}$.
    
    \end{proof}

    From this it follows that \cref{alg:exp} will indeed always output a feasible solution of $\ell$ elements, due to \cref{eq:exp-alg} when $i=\ell$.
    
    We now turn to showing a lower bound on the quality in terms of $f(\text{OPT})$ of the solution output by \cref{alg:exp}.
    
    \begin{lemma}
    The solution output by \cref{alg:exp} has value at least $(1/2-\eta)\cdot f(\text{OPT})$.
    \end{lemma}
    
    \begin{proof}
    We prove the following statement by induction over $i$:
    \begin{equation}\label{eq:exp-alg-induction}
    2\cdot f(\{s_1,\ldots,s_i\})+f(\{o_{i+1},\ldots,o_\ell\}|\{s_1,\ldots,s_i\})+\frac{i\eta\gamma}{\ell}\ge f(\text{OPT}).
    \end{equation}
    
    The base case of $i=0$ holds trivially, and by substituting in $i=\ell$, we get the statement of the Lemma.
    
    For simplicity denote $S=\{s_1,\ldots,s_{i-1}\}$ and $O=\{o_{i+1},\ldots,o_\ell\}$. To prove the inductive step, it suffices to show that the change in the left hand size of \cref{eq:exp-alg-induction} is positive when moving form $i-1$ to $i$:
    \begin{align*}
        \text{LHS}_i-\text{LHS}_{i-1}&=2\cdot f(S+s_i)-2f(S)+f(O|S+s_i)-f(O+o_i|S)+\frac{\eta\gamma}\ell\\
        &\ge2\cdot f(s_i|S)-f(o_i|S)-f(s_i|S)+\frac{\eta\gamma}\ell\\
        &\ge f(s_i|S)+\frac\eta\ell-f(o_i|S)\\
        &\ge0,
    \end{align*}
    since we know that $f(o_i|S)$ and $f(e|S)$ are both in $[\theta\eta\gamma/\ell,(\theta+1)\eta\gamma/\ell)$ from Lines~\ref{line:function-oracle} and~\ref{line:function-check}.
    
    Finally, taking \cref{eq:exp-alg-induction} with $i=\ell$ gives us
    \begin{align*}
    2\cdot f(\{s_1,\ldots,s_\ell\})+\eta\gamma \ge f(\text{OPT}),
    \end{align*}
    and therefore
    \begin{align*}
    f(\{s_1,\ldots,s_\ell\})\ge f(\text{OPT})\cdot(1/2-\eta\cdot(1+\eta)/2)\ge f(\text{OPT})\cdot(1/2-\eta).
    \end{align*}

    \end{proof}

    This concludes the proof of correctness of \cref{alg:exp} in the presence of the four oracles (cardinality oracle, color oracle, function oracle, and matroid oracle). We are now ready to prove \cref{thm:alg-exp}.

    \begin{lemma}\label{lemma:pos-prob}
    There exists a \emph{randomized} single-pass streaming algorithm using $O(k)$ memory, and outputting a $1/2-\eta$-approximately optimal feasible solution with positive probability, while consuming $O(k^2 + k \log C)$ bits of randomness.
    \end{lemma}
    
    \begin{proof}
    \cref{alg:exp} is such an algorithm when replacing the three oracles with uniformly random choices. Indeed, it produces the correct output with positive probability (when all random choices happen to be correct).
    
    The cardinality oracle is called only once and chooses between $k$ options, so a random implementation consumes $\log k$ random bits. The color oracle is called at most $k$ times and chooses between $C$ options, so a random implementation consumes $O(k\log C)$ random bits. The function oracle is called at most $k$ times and chooses between $O(k/\eta)$ options. This is because $\theta$ is at least $0$ and at most $\ell/\eta\le k/\eta$, since $\gamma\ge f(\text{OPT})\ge f(o_i)\ge f(o_i|S)$. Therefore, a random implementation consumes $O(k\log(k/\eta))$ random bits. The matroid oracle is called at most $k^2 + k$ times; at most $k + 1$ times every iteration of the for loop in Line~\ref{line:main-loop}. This is because every time it is called and returns true (that is $\{s_1,\ldots,s_{i-1},e,o_{i+1},\ldots,o_\ell\}\in\mathcal I$), the current iteration of the loop is terminated; every time it is called and returns false, $T$ is incremented, and since $T\in\mathcal I$, it can have size at most $k$. Therefore, a random implementation of this consumes $k^2 +k$ random bits.
    
    In total this is $O(\log k+ k^2+k\log(k/\eta)+k\log C)=O(k^2 + k\log C)$ random bits.
    \end{proof}
    
    \begin{proof}[Proof of \cref{thm:alg-exp}]
    We simply run $2^{O(k^2+k \log C)}$ parallel copies of the algorithm guaranteed by \cref{lemma:pos-prob}, each with a different stream of bits as randomness. At least one is guaranteed to succeed. We can then find and return the highest valued feasible set output by the algorithms.
    
    However, all versions of \cref{alg:exp} assume access to $\gamma$ with the guarantee that $f(\text{OPT})\in[(1-\eta)\cdot\gamma, \gamma]$. For the purposes of this proof, we call the above algorithm the $\gamma$-dependent algorithm; it satisfies the requirements postulated by \cref{thm:alg-exp}, but only under the condition that $\gamma$ is set correctly. We will now show how to drop this requirement while losing a $O(\log\Delta)$ factor in the memory complexity. (What follows is standard technique often used in the literature in the context of streaming submodular maximization.) We again run multiple copies of the $\gamma$-dependent algorithm, with different guesses of $\gamma$. In fact, we run a copy for $\gamma=(1-\eta)^t$ for every $t\in\mathbb Z$, thus guaranteeing that in at least one of the cases $f(\text{OPT})\in[(1-\eta)\cdot\gamma, \gamma]$ is satisfied.
    
    Although this is potentially an infinite number of parallel copies, we show that all but $O(\log(k\Delta))$ of them fall into one of two classes, such that copies within the same class look identical to each other; thus we require only $2^{O(k^2)}\cdot\log\Delta$ memory in the end. Recall that $\Delta$ is the ratio between the value of the larges and smallest (non-zero) elements on the stream, that is $\Delta=\max f(e)/\min_{f(e)\neq0}f(e)$. Let the largest and smallest elements be $e_\text{max}$ and $e_\text{min}$ respectively, such that $\Delta = f(e_\text{max})/f(e_\text{min})$.
    
    For values of $\gamma$ that are less than $f(e_\text{min})/2$, an element $e$ can only pass the filter at Line~\ref{line:function-check} if $f(e)=0$. This can be proven by induction. As long as $S$ contains only elements with $f$-value $0$, $f(S)=0$, and $\forall e:\ f(e|S)=f(e)$. For any $e$ such that $f(e)>0$ it follows that $f(e|S)\ge f(e_\text{min})>(\theta+1)\eta\gamma/\ell$ so $e$ does not pass the filter at Line~\ref{line:function-check}. As a result all copies of the $\gamma$-dependent algorithm with $\gamma\le f(e_\text{min})/2$ look identical to each other and can be stored as one.
    
    For values of $\gamma$ that are greater than $f(e_\text{max})\cdot k/\eta$ it is also true that all copies of the $\gamma$-dependent algorithm look identical. In this case, if $\theta$ is anything other than $0$ on Line~\ref{line:function-oracle}, all elements will be filtered out on Line~\ref{line:function-check}, since $\forall e:\ f(e)<\gamma\eta/\ell$. On the other hand, if $\theta$ is $0$ on Line~\ref{line:function-oracle}, \emph{all} elements $e$ pass the filter on Line~\ref{line:function-check} for the same reason. Therefore, when $\gamma>f(e_\text{max})\cdot k/\eta$, the exact value of $\gamma$ is irrelevant to the execution of the $\gamma$-dependent algorithm, and all such copies can be stored as one.
    
    In summary, copies of the $\gamma$-dependent algorithm for $\gamma$'s over $f(e_\text{max})\cdot k/\eta$ as well as $\gamma$'s under $f(e_\text{min})/2$ are stored as though they constituted only two total copies of the $\gamma$-dependent algorithm. (This can be done  without foreknowledge of $f(e_\text{max})$, $f(e_\text{min})$ or even $\Delta$.) All other copies of the algorithm are stored explicitly --- a total of $O(\log(k\Delta))$ copies. Once again, the correct solution among all possibilities can be selected by simply picking the largest $f$-valued \emph{feasible} set. The total memory complexity is $2^{O(k^2)}\cdot\log\Delta$.

    \end{proof}

\section[Proof of the other Theorem]{Proof of \cref{hardness-streaming-main}} \label{sec:proof_of_hardness-streaming-main}
        Towards a contradiction assume that $\mathcal{A}$ is an algorithm as in the statement of \cref{hardness-streaming-main}.  We then describe how to construct an algorithm $\mathcal{B}$ to find a perfect matching. Consider any instance of the perfect bipartite matching streaming problem, and let $\left<(l_1, r_1), (l_2, r_2), \ldots (l_{|E|}, r_{|E|}) \right>$ denote the stream of edges of a bipartite graph $G(L\cup R, E)$. 
        Define the following matroid constraint on $E$: a subset of edges is independent if it has at most one edge incident to any left vertex $l \in L$. Note that this is a partition matroid, and its rank $k = n$, as we can assume that each left vertex has at least one edge (otherwise there is no perfect matching). Moreover, we use the fairness constraint to ensure that {\em exactly} one edge incident to each vertex $r \in R$ is selected; we have $C = n$. 
        With these two constraints on the set of edges, we have that any solution $S\subseteq E$ is feasible if and only if $S$ is a perfect matching. The submodular function $f$ does not play a role in this reduction and can be defined arbitrarily. $\mathcal{B}$ can simulate the behavior of algorithm $\mathcal{A}$ on the edge set with the constraints defined above and returns that there exists a perfect matching if and only $\mathcal{A}$ return that there exists a feasible solution. This contradicts \cref{matching-hardness} and concludes the proof.

\end{document}

%% file: figs/figure.tex
\begin{figure*}[t] 
\begin{subfigure}{.33\textwidth}
  \centering
\includegraphics[trim=20 150 0 150, clip, scale=.25]{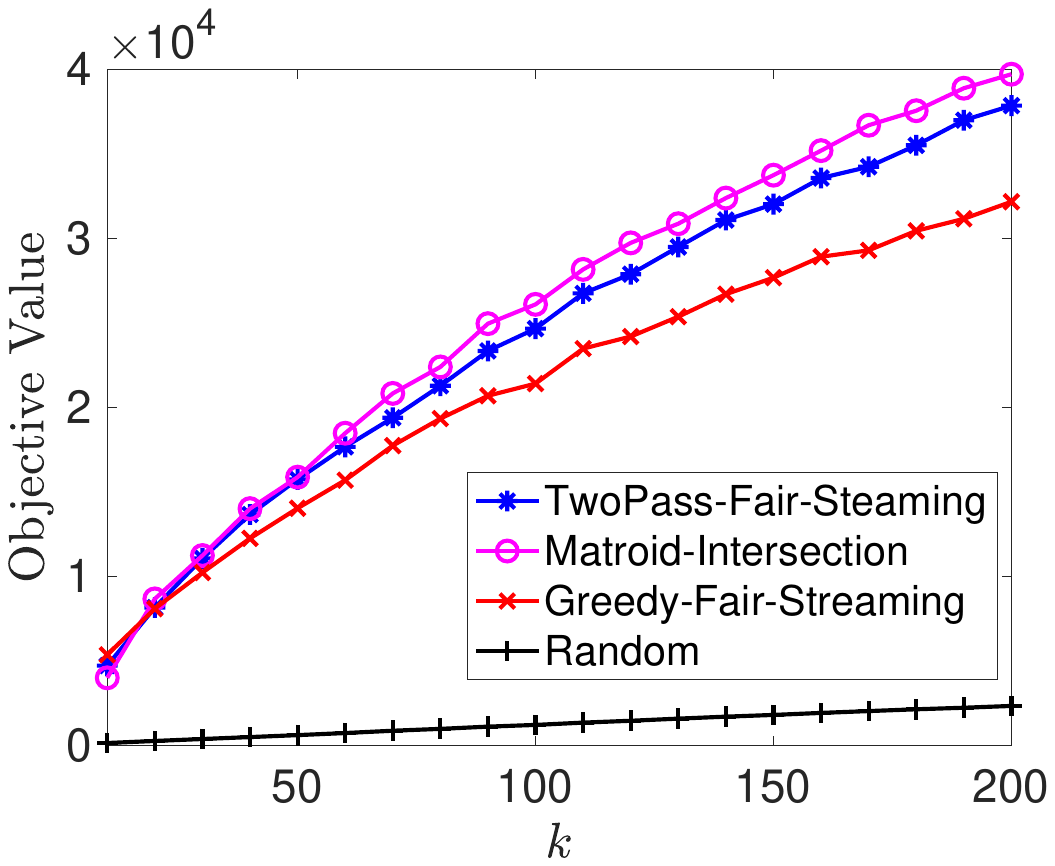} 
\vspace*{-10mm}\caption{\small\label{fig:coverage-obj} Maximum coverage}
\end{subfigure}
\begin{subfigure}{.33\textwidth}
  \centering
\includegraphics[trim=20 150 0 150, clip, scale=.25]{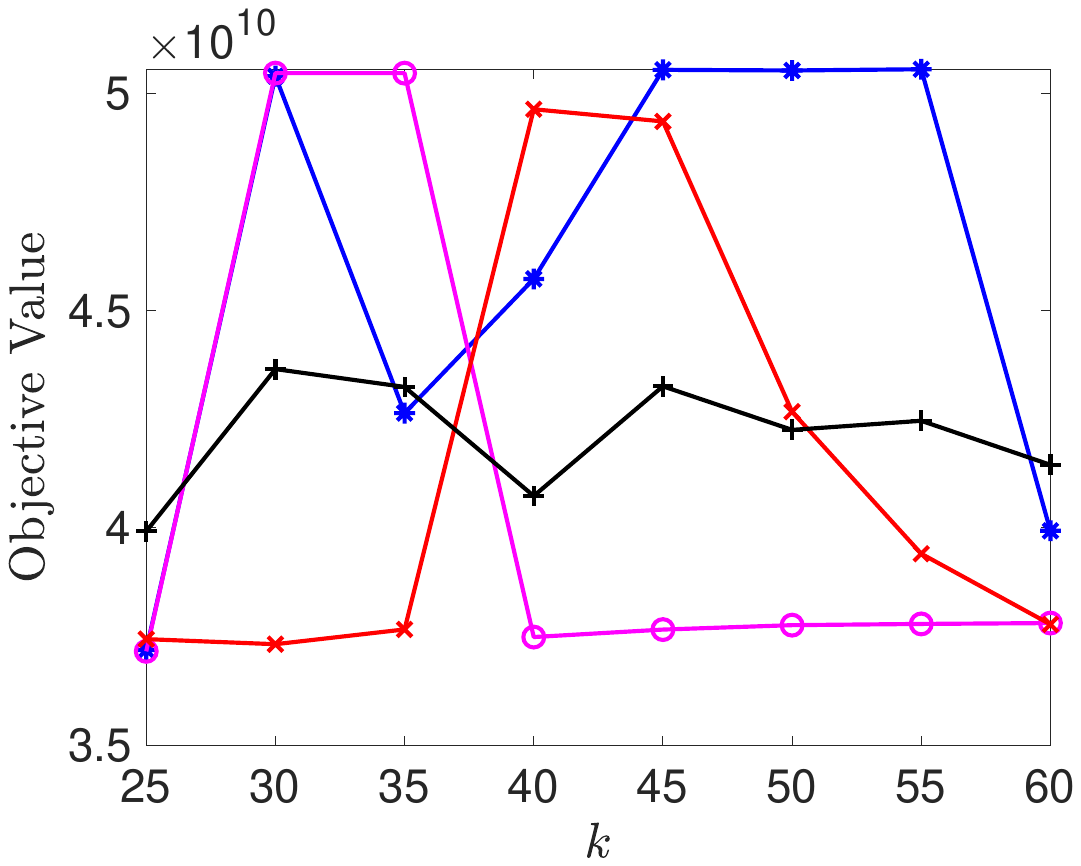} 
\vspace*{-10mm}\caption{\small\label{fig:bank-obj} Exemplar clustering}
\end{subfigure}
\begin{subfigure}{.33\textwidth}
  \centering
\includegraphics[trim=20 150 0 150, clip, scale=.25]{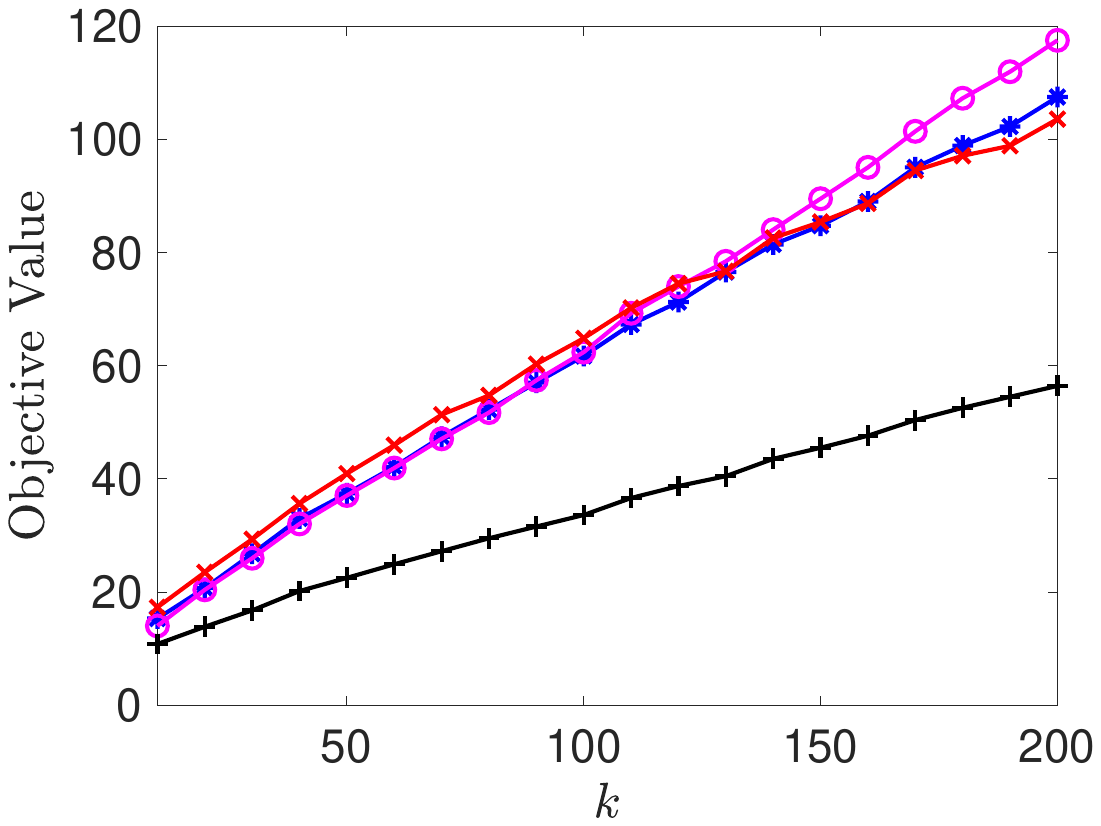} 
\vspace*{-10mm}\caption{\small\label{fig:movies-obj}  Movie recommendation}
\end{subfigure} \vspace{-15pt}\\
\begin{subfigure}{.33\textwidth}
  \centering
\includegraphics[trim=20 150 0 150, clip, scale=.25]{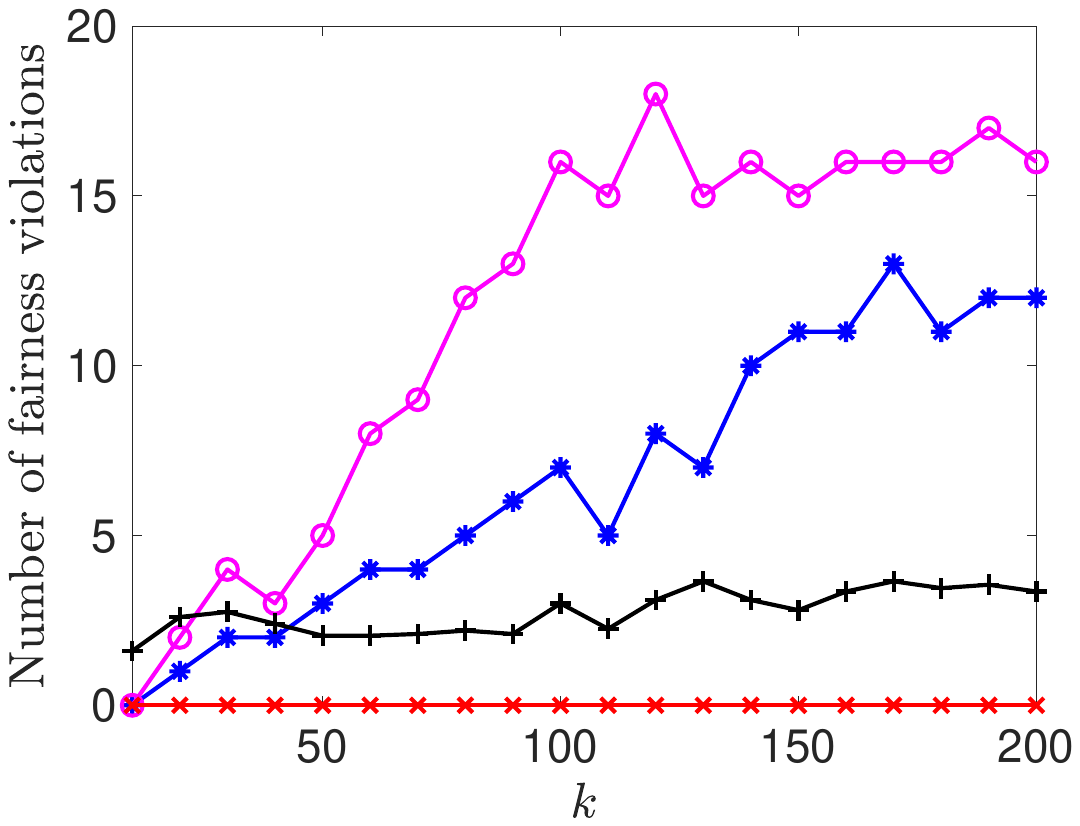}
\vspace*{-10mm}\caption{\small\label{fig:coverage-err} Maximum coverage}
\end{subfigure}
\begin{subfigure}{.33\textwidth}
  \centering
\includegraphics[trim=20 150 0 150, clip, scale=.25]{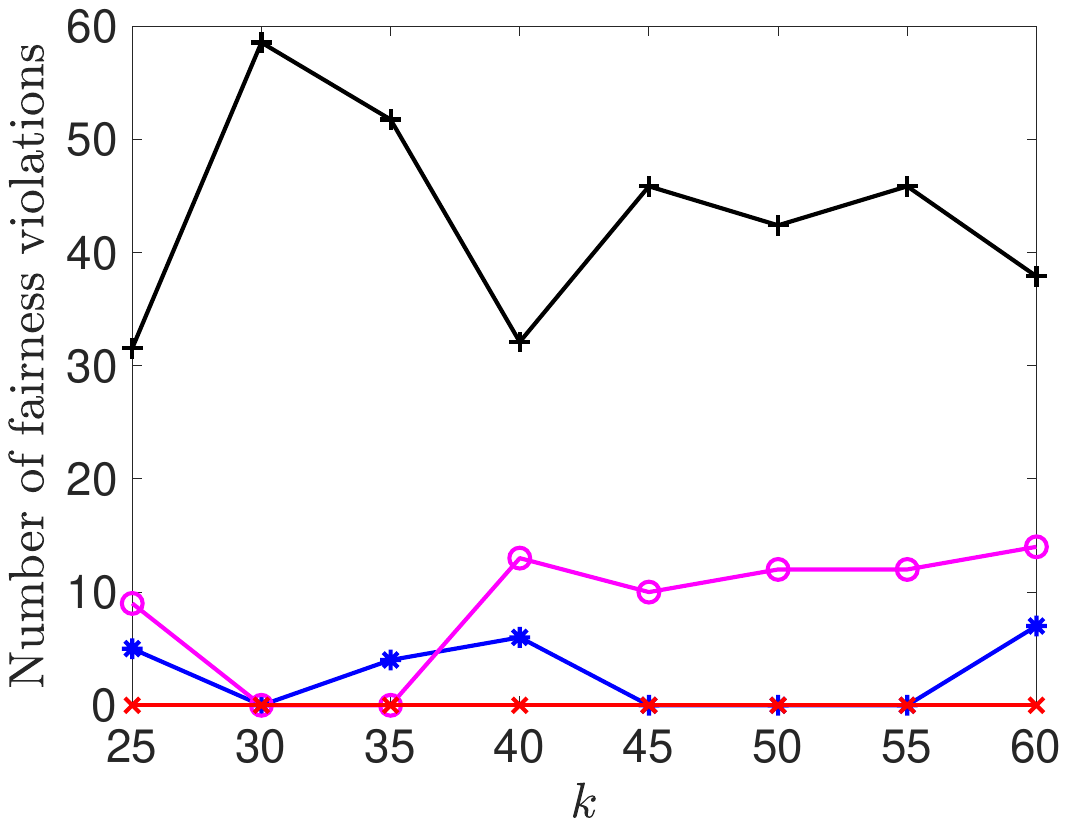}
\vspace*{-10mm}\caption{\small\label{fig:bank-err} Exemplar clustering}
\end{subfigure}
\begin{subfigure}{.33\textwidth}
  \centering
\includegraphics[trim=20 150 0 150, clip, scale=.25]{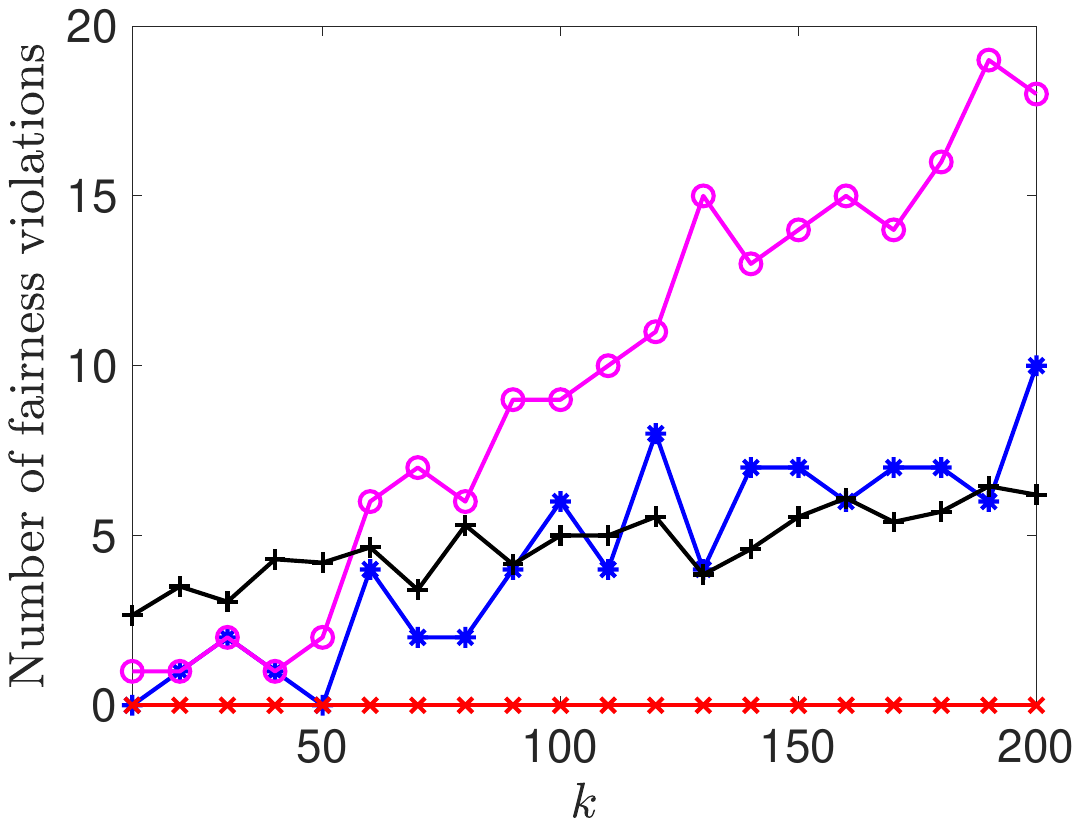}
\vspace*{-10mm}\caption{\small \label{fig:movies-err}  Movie recommendation}
\end{subfigure}
\caption{ \label{fig:results} Objective values (a,b,c) and number of fairness violations (d,e,f) on the three applications.}
\end{figure*}